\documentclass[11pt,lettersize]{article}

\usepackage{times}
\usepackage{fullpage}

\usepackage{amsfonts}
\usepackage{amsmath}
\usepackage{amssymb}
\usepackage{amsthm}
\usepackage{bm}

\usepackage{graphicx}
\usepackage{booktabs}
\usepackage[subrefformat=parens]{subcaption}
\usepackage{multirow}

\usepackage{color}
\usepackage{xcolor}

\usepackage[backref=page]{hyperref}
\hypersetup{colorlinks=true,
            citecolor=blue,
            linkcolor=blue,
            urlcolor=blue,}

\newtheorem{assumption}{Assumption}

\newtheorem{thm}{Theorem}
\newtheorem{lem}{Lemma}
\newtheorem{rmk}{Remark}

\newtheorem{prop}{Proposition}

\usepackage{alphalph}

\usepackage[numbers,square]{natbib}

\usepackage{authblk}
\title{Investigating Self-Supervised Image Denoising with Denaturation}

\author[1]{Hiroki Waida\footnote{H.~Waida did this work while he was an internship student at Fujitsu Ltd.}}
\author[2]{Kimihiro Yamazaki}
\author[3]{Atsushi Tokuhisa}
\author[2]{Mutsuyo Wada}
\author[2,4]{Yuichiro Wada\footnote{Corresponding author. E-mail: wada.yuichiro@jp.fujitsu.com}}

\affil[1]{\normalsize
            Department of Mathematical and Computing Science, Institute of Science Tokyo,
            2-12-1 Ookayama, Meguro-ku, 
            Tokyo,
            152-8550, 
            Japan
            }

\affil[2]{\normalsize
            Fujitsu Limited,
            4-1-1 Kamikodanaka, Nakahara-ku, Kawasaki-shi, 
            Kanagawa,
            211-8588, 
            Japan
            }

\affil[3]{\normalsize
            RIKEN Center for Computational Science,
            7-1-26 Minatojima-minami-machi, Chuo-ku, Kobe, 
            Hyogo,
            650-0047,
            Japan
            }

\affil[4]{\normalsize
            RIKEN Center for Advanced Intelligence Project,
            Nihonbashi 1-chome Mitsui Building, 15th floor, 1-4-1 Nihonbashi, Chuo-ku, 
            Tokyo,
            103-0027, 
            Japan
            }

\date{}

\begin{document}

\maketitle

\begin{abstract}
Self-supervised learning for image denoising problems in the presence of denaturation for noisy data is a crucial approach in machine learning. However, theoretical understanding of the performance of the approach that uses denatured data is lacking. To provide better understanding of the approach, in this paper, we analyze a self-supervised denoising algorithm that uses denatured data in depth through theoretical analysis and numerical experiments. Through the theoretical analysis, we discuss that the algorithm finds desired solutions to the optimization problem with the population risk, while the guarantee for the empirical risk depends on the hardness of the denoising task in terms of denaturation levels. We also conduct several experiments to investigate the performance of an extended algorithm in practice. The results indicate that the algorithm training with denatured images works, and the empirical performance aligns with the theoretical results. These results suggest several insights for further improvement of self-supervised image denoising that uses denatured data in future directions. 
\end{abstract}

\section{Introduction}
\label{sec:introduction}
Image denoising is a task to predict the clean image from an observation perturbed by some random noise, which is recognized as an important topic for decades~\cite{dabov2007image,zhang2017beyond,lehtinen2018noise}.
\cite{zhang2017beyond,lehtinen2018noise} provide significant improvement for this task by making use of the following components: \cite{zhang2017beyond} show the power of deep learning for supervised denoising tasks, and \cite{lehtinen2018noise} focus on a strategy that learns noise patterns from multiple pairs of images whose clarity of the signals may not be necessarily high enough.
After \cite{lehtinen2018noise} has presented this self-supervised strategy, the combination of these components has been actively studied: for instance, \cite{krull2019noise,huang2021neighbor} proposed alternative techniques for generating a pair of noisy images from a single observation, and \cite{quan2020self,mansour2023zero} introduced denoising methods that learn the noise pattern from a test image itself in a self-supervised manner.
The method of~\cite{lehtinen2018noise} is theoretically well-investigated in~\cite{zhussip2019extending,huang2021neighbor}.
\cite{zhussip2019extending} show that the empirical method studied by~\cite{lehtinen2018noise} implicitly utilizes a theoretical relation between self-supervised denoising and supervised denoising.
\cite{huang2021neighbor} also consider theory on the method of~\cite{lehtinen2018noise} under general assumptions for noise distributions.

From practical viewpoints, denoising of images collected from a device is a useful technique before handling other operations using the data. For instance, Cryo-Electron Microscopy (Cryo-EM) collects 2D projection images of target biomolecules~\cite{Earl2017-if}.
As~\cite{Lyumkis2019-ds} point out, those images are not only noisy but also denatured due to an electron beam. It is important to denoise from the 2D projection images, since it is shown that the denoised Cryo-EM images can help to reconstruct 3D structures of the biomolecules~\cite{scheres2012relion,punjani2017cryosparc}.
We note that the reconstructed structures have a potential to lead biological significance, e.g., a biological finding to develop anti-cancer drugs~\cite{Yuan2022structural}. As another example, Magnetic Resonance Imaging (MRI) collects noisy images with kind of denaturation, and the corresponding self-supervised denoising problems are investigated in~\cite{xu2021deformed,gan2022deformation}.
Considering this background, understanding the ability of self-supervised denoising using denatured images is a crucial topic.

On the other hand, the seminal work by~\cite{lehtinen2018noise} designed a statistical modeling for a noisy image whose conditional expectation with respect to the clean counterpart coincides with the clean one. Therefore, the modeling of~\cite{lehtinen2018noise} cannot apply to the situation in which the clean image can be denatured by some transformation.
In the field of self-supervised denoising, several works by~\cite{ehret2019joint,gan2022deformation} have shown the empirical advantages of their proposed methods for noisy images with denaturation.
However, it is not theoretically investigated whether some method based on the methodology of~\cite{lehtinen2018noise} can deal with denoising problems in the presence of some general and unknown transformation to training images.
Addressing this issue is important to fully utilize the methodology of~\cite{lehtinen2018noise} under general settings.

In this paper, we address the above issue from both the theoretical and empirical viewpoints.
We study a self-supervised denoising algorithm for estimating the fixed target image of high-resolution from denatured training images contaminated by noise.
In particular, our approach focuses on a simple extension from the framework called \emph{Noise2Noise}~\cite{lehtinen2018noise} to facilitate learning by making use of the common features among the denatured images.
The contributions of this paper are summarized below:
\begin{itemize}
    \item We show that similarly to the theoretical results presented in~\cite{zhussip2019extending,huang2021neighbor}, the algorithm extended from~\cite{lehtinen2018noise} has the theoretical guarantees for a population loss minimizer. 
    \item In addition, we derive a statistical guarantee for an empirical risk minimizer. The guarantee holds under the non-asymptotic statistical theory, which has not been addressed yet in the context of self-supervised denoising. The result reveals both the quantitative performance and limitations of the algorithm.
    \item Building on the theoretical results, we instantiate the self-supervised denoising algorithm termed \textit{Denatured-Noise2Noise} (DN2N) to investigate the empirical aspects of the algorithm.
    In the experiment that uses toy datasets with denaturation and noise, we confirm the consistency between the theoretical analysis and empirical performance of DN2N. We also show efficiency of DN2N for the following three benchmark datasets: MRI, Cryo-EM, and Fluorescence Microscopy (FM)~\cite{lichtman2005fm} image datasets. 
\end{itemize}

\section{Preliminaries}
\label{sec:self-supervised denoising for denatured noisy images}
We begin by formulating the problem setting and reviewing some background on self-supervised denoising.

\subsection{Problem Setting}
\label{subsec:problem setting}
The aim of this study is to analyze a self-supervised image denoising algorithm extended from the basic framework of~\cite{lehtinen2018noise} to leverage denatured noisy images.
To this end, we present the fundamental problem setting in this paper.

Let $(\Omega,P)$ be a probability space, and let $\tau:\Omega\to(0,T]$ be a random variable indicating time. Let $\bm{y}:\Omega\times[0,T]\to\mathbb{R}^{d}$ be a continuous-time stochastic process.
This stochastic process represents the time-series of denatured noisy images.
We aim to make use of information derived from the denatured images to learn a neural-network-based model $f$ to predict as $f(\bm{y}_{0})\approx \bm{x}_{0}$, where $\bm{x}_{0}$ is
unknown during both the training and test phases.
We note that we consider the fixed design for the estimation problem: in fact, the target image $\bm{x}_{0}$ is fixed throughout the setting, and the training images are utilized to estimate the realization $\bm{x}_{0}$.

The difficulty of this setup arises from the condition that the sequence is denatured, meaning that for a mapping $\phi: \mathbb{R}^{d}\times [0,T]\to\mathbb{R}^{d}$, the clean image $\bm{x}_{t}$ corresponding to $\bm{y}_{t}$ is represented by $\phi_{t}(\bm{x}_{0})$, where note that $\phi_{0}(\bm{x}_{0})=\bm{x}_{0}$ (the formal definition is introduced in Assumption~\ref{assumption:assumptions} in Section~\ref{subsec:a guarantee for population minimizers}).
For instance, when we choose to model the noise $\bm{\varepsilon}_{t}$ as the multivariate standard normal distribution, the statistical modeling for $\bm{y}_{t}$ may be written as
\begin{align}
\label{eq:statistical modeling}
\bm{y}_{t}=\phi_{t}(\bm{x}_{0})+\bm{\varepsilon}_{t}\quad\forall t\in [0,T].
\end{align}
When $T=0$, \eqref{eq:statistical modeling} is reduced to a modeling studied by~\cite{lehtinen2018noise}.
Besides, the modeling recovers that of~\cite{xu2021deformed} when $\phi_{t}$ represents the deformation transforms for images.
From statistical viewpoints, this problem may be resolved by estimating $\phi_t$ directly under the regression modeling.
Unfortunately, this strategy cannot apply to our setting since the clean data $\bm{x}_{0}$ is unknown as in the literature on self-supervised denoising~\cite{lehtinen2018noise,quan2020self,xu2021deformed}.

\subsection{Background from Technical Viewpoints}
\label{subsec:theoretical guarantees of denoising for denatured noisy images}
We next discuss some background of the previous self-supervised denoising methods to see what to be investigated for achieving our goal.

Several self-supervised denoising methods using many images have been proposed in a line of research~\cite{lehtinen2018noise,krull2019noise,batson2019noise,moran2020noisiser,huang2021neighbor,pang2021recorrupted,wang2022blind,lee2022apbsn,vaksman2023patch,wang2023noise2info,pan2023random}.
Several existing frameworks proposed by~\cite{krull2019noise,batson2019noise,huang2021neighbor} are built on a method proposed by~\cite{lehtinen2018noise}.
Formally, \cite{lehtinen2018noise} formulate their framework as minimization of the objective 
\begin{align*}
\mathbb{E}_{\bm{y},\bm{y}'}[\|g(\bm{y})-\bm{y}'\|_{2}^{2}],
\end{align*}
where $\bm{y}$ and $\bm{y'}$ are random noisy images, and $g$ is a function that predicts the clean image $\bm{x}_{0}$.
Note that~\cite{lehtinen2018noise} consider a more general loss function including the above one in their method (see Eq.(2) in~\cite{lehtinen2018noise}), while we primarily focus on the $\ell_{2}$ distance in our work.
Here, let us rewrite the objective function of~\cite{lehtinen2018noise} as the expectation of conditional expectation for $\bm{x}_{0}$, that is, 
\begin{align*}
\mathbb{E}_{\bm{y},\bm{y}'}[\|g(\bm{y})-\bm{y}'\|_{2}^{2}]=\mathbb{E}_{\bm{x}_{0}}[\mathbb{E}_{\bm{y},\bm{y}'}[\|g(\bm{y})-\bm{y}'\|_{2}^{2}|\bm{x}_{0}]].
\end{align*}
Let us focus on the loss $\mathbb{E}_{\bm{y},\bm{y}'}[\|g(\bm{y})-\bm{y}'\|_{2}^{2}|\bm{x}_{0}]$.
In fact, Theorem~1 in~\cite{huang2021neighbor} implies that under some mild conditions, this loss function has the learning theory guaranteeing the equality
\begin{align*}
    \mathbb{E}_{\bm{y},\bm{y}'}[\|g(\bm{y})-\bm{y}'\|_{2}^{2}|\bm{x}_{0}]=
    \mathbb{E}_{\bm{y}}[\|g(\bm{y})-\bm{x}_{0}\|_{2}^{2}|\bm{x}_{0}].
\end{align*}
However, it is pointed out by~\cite{krull2019noise} that preparing pairs of random variables $(\bm{y},\bm{y}')$ is not realistic in practical situations.
We also face an issue similar to~\cite{krull2019noise} when attempting to apply the method of~\cite{lehtinen2018noise} directly to our setting, since we consider a case that the sufficient amount of such pairs may not be available.

If we substitute $\bm{y}'$ for a denatured noisy image $\bm{y}_{\tau}$ whose clean image is represented as $\phi_{\tau}(\bm{x}_{0})$, then following a similar idea to~\cite{lehtinen2018noise}, we can define the loss $\mathbb{E}[\|g(\bm{y}_{0})-\bm{y}_{\tau}\|_{2}^{2}|\bm{x}_{0}]$.
A variant of this loss has already been employed in the context of biological image denoising~\cite{bepler2020topaz}.
Recall that \cite[Theorem~1]{huang2021neighbor} imply under some conditions that optimizing the loss is equal to minimizing $\mathbb{E}[\|g(\bm{y}_{0})-\phi_{\tau}(\bm{x}_{0})\|_{2}^{2}|\bm{x}_{0}]$.
This implies that learning with this loss function is affected by the changes in the denatured images, which results in worsening the quality of prediction after the stage of training.
This loss of information also makes it hard to apply a recent approach explored by~\cite{chen2023multi} to this setting.

We also mention that single image denoising method have also been proposed in~\cite{lempitsky2018deep,quan2020self,lequyer2022noise,ta2022poisson,mansour2023zero}.
\cite{quan2020self} propose a self-supervised method that uses only the noisy image $\bm{y}_{0}$ to learn denoising models.
\cite{lequyer2022noise} propose a fast method with a theoretical background.
However, \cite{lequyer2022noise} builds their theory on the assumption that a noisy image is defined as the summation of a clean image and some additive random variable.
Meanwhile, we aim to investigate the performance of self-supervised denoising under more general settings.
\cite{mansour2023zero} also improve the running time by considering to leverage the framework of~\cite{lehtinen2018noise}.
Note that different from~\cite{mansour2023zero}, we study the properties of the methodology of~\cite{lehtinen2018noise} in a scenario where denatured images are available for training.

\section{A Time-Aware Denoising Loss}
\label{sec:denatured-noise2noise algorithm}
Motivated by the discussion in Section~\ref{subsec:theoretical guarantees of denoising for denatured noisy images}, we consider to utilize denatured noisy image to perform self-supervised denoising without loss of much information from $\bm{x}_{0}$.
Following a similar idea to~\cite{xu2021deformed,ehret2019joint,gan2022deformation}, we consider to extend a framework of~\cite{lehtinen2018noise} to suppress the loss of information due to the general mapping $\phi_{t}$ and leverage the remained information for enhancing training.
As we will see in the later section, this approach also enables to ensure a theoretical guarantee for training (see Section~\ref{sec:theoretical analysis}).

\subsection{Definition}
\label{subsec:definition}
The basic strategy of our analysis is to introduce additional time variable to a self-supervised denoising framework.
Let $f:\mathbb{R}^{d}\times[0,T]\to\mathbb{R}^{d}$ be a training model, where we aim to train $f$ such that $f(\bm{y}_{0},\tau)\approx\phi_{\tau}(\bm{x}_{0})$.
We first focus on the following self-supervised learning objective extended from~\cite{lehtinen2018noise}:
\begin{align}
\label{eq:time-aware denoising model}
\mathbb{E}_{\bm{y},\tau}[\|f(\bm{y}_{0},\tau)-\bm{y}_{\tau}\|_{2}^{2}|\bm{x}_{0}].
\end{align}
The objective function~\eqref{eq:time-aware denoising model} is a simple extension from the framework called \emph{Noise2Noise}~\cite{lehtinen2018noise}, where we add time information due to our problem setting.
In fact, when $f$ is defined as standard neural networks, the last entry of the vector $(\bm{y}_{0},0)$ should be ignored in the matrix operation of the first layer.
Therefore, if $T=0$ holds as a special case, then the loss~\eqref{eq:time-aware denoising model} recovers the original loss function proposed in~\cite{lehtinen2018noise}.

After training, we predict the target clean image $\bm{x}_{0}$ by letting $t=0$ in its input, namely $f(\bm{y}_{0},0)$.
Intuitively, the denatured information in the images $\bm{y}_{t}$, $t\in (0,T]$ are used for learning how to trace back to $\bm{x}_{0}$ from the denatured but clean image $\phi_{t}(\bm{x}_{0})$.
We note that a similar setting for denoising tasks is considered in~\cite{mildenhall2018burst}, where \cite{godard2018deep,ehret2019joint,dudhane2022burst} also study utilization of deep learning in the same context as~\cite{mildenhall2018burst}.
One of the notable strengths of the extended approach in \eqref{eq:time-aware denoising model} is that we can make use of denatured images directly without combining additional techniques.
In fact, let us recall that \cite{krull2019noise} use some operations such as some data augmentation technique, and \cite{huang2021neighbor} introduce a method to construct new images by drawing pixels randomly from the noisy image.
Whereas in~\eqref{eq:time-aware denoising model}, by utilizing information of similarity existing in the sequence of images $\{\bm{y}_{t}\}$ instead, we can reduce the cost drastically.

As a remark, the loss function~\eqref{eq:time-aware denoising model} is related to some losses proposed in several previous works~\cite{xu2021deformed,ehret2019joint,gan2022deformation,ho2020denoising,xie2023diffusion}.
We discuss this point in Section~\ref{subsec:comparison with related work}.

\subsection{Relations to the Previous Work}
\label{subsec:comparison with related work}
To the best of our knowledge, among the existing literature on denoising using sequences of noisy images, the work by~\cite{xu2021deformed} is closely relevant to our problem setting.
\cite{xu2021deformed} utilize a framework referred to as \emph{Self2Self}~\cite{quan2020self} to denoise deformed noisy images and then correct deformation of the predicted images by applying a sampling technique for grids in images.
\cite{xu2021deformed} empirically verify that the performance of their method overwhelms that of~\cite{quan2020self}.
However, the theoretical analysis showing what will be reconstructed as the final output from the model trained according to the method of~\cite{xu2021deformed} is an open problem.
We note that the target of our work is the methodology of \emph{Noise2Noise}~\cite{lehtinen2018noise}, different from~\cite{xu2021deformed}.
Moreover, \cite{xu2021deformed} investigate in a setting that underlying clean images are deformed.
Meanwhile, in this paper, we address a broader problem by treating $\phi_{t}$ as a general mapping, which can be non-linear transformations.

We also discuss the previous works~\cite{ehret2019joint,gan2022deformation} related to our approach.
\cite{ehret2019joint} introduce a method to learn how to remove both mosaic and noise components simultaneously, where their approach minimizes a norm function so that a model can learn how to transform an image into others.
The main difference from~\cite{ehret2019joint} is that we additionally input the time variable for the model $f$.
\cite{gan2022deformation} tackle the self-supervised denoising with deformed noisy medical images by extending a framework of~\cite{lehtinen2018noise}.
However, \cite{gan2022deformation} mainly focus on the empirical aspects of their method.

Here we also note that~\eqref{eq:time-aware denoising model} has similarity to the method called \emph{Denoising Diffusion Probabilistic Model}~\cite{ho2020denoising}.
\cite{ho2020denoising} utilize the mechanism of denoising to generate high quality images.
Meanwhile, \eqref{eq:time-aware denoising model} has the following main difference in its mechanism from the loss of~\cite{ho2020denoising}: \eqref{eq:time-aware denoising model} is built on the idea that the reconstruction $f(\bm{y}_{0},\tau)$ is made close to the denatured noisy image $\bm{y}_{\tau}$, while \cite{ho2020denoising} consider to fill the gap between prediction and noise itself.
Note that the combination of MRI denoising and the generative model of~\cite{ho2020denoising} is also recently paid attention in~\cite{xiang2023ddm}.

The authors of~\cite{xie2023diffusion} build on the above generative model of~\cite{ho2020denoising} to propose a method for a supervised image denoising task (see, e.g.,~\cite{zhang2017beyond} for supervised image denoising). In~\cite{xie2023diffusion}, a denoiser $f$ parameterized by a deep neural network is trained by minimizing $\|f(\widetilde{\bm{x}}_t, t) - \bm{x}_0\|_2^2$ w.r.t. $f$, where $\bm{x}_0$ is a clean image, and $\widetilde{\bm{x}}_t$ is the $t$-th noise-contaminated image defined from $\bm{x}_0$, which is based on the way referred to as \emph{forward process}~\cite{ho2020denoising}, and after the training, the predicted clean image for the noisy one is computed based on the approach referred to as \emph{reverse process}~\cite{ho2020denoising}.
The main differences of the loss~\eqref{eq:time-aware denoising model} to that of~\cite{xie2023diffusion} are that we consider the self-supervised learning setting where the clean image is unavailable, and in our setting, the denatured image $\bm{y}_{t}$ is observed as data, not sampled by the approach of~\cite{ho2020denoising}.

\section{Theoretical Analysis}
\label{sec:theoretical analysis}
A question to the previous section is whether the framework has some theoretical guarantee.
In this section, we first give a positive answer to this question.
Note that in this section, the random variables $\{\bm{y}_{0,j}\}_{j=0}^{M}$, $\{\tau_{i}\}_{i=0}^{N}$, and $\{\bm{y}_{t,j}\}_{j=0}^{M}$ are identically distributed within each sequence of random variables, and all the random variables are independent, where $\bm{y}_{0,0}=\bm{y}_{0},\tau_{0}=\tau$, and $\bm{y}_{t,0}=\bm{y}_{t}$.
Furthermore, the clean image $\bm{x}_{0}$ is fixed.

\subsection{On Population Risk Minimizers}
\label{subsec:a guarantee for population minimizers}
The theoretical guarantee of the time-aware denoising algorithm is mainly based on the following conditions.
\begin{assumption}[Assumptions on data]
\label{assumption:assumptions}
The following conditions hold:
\begin{enumerate}
    \item[\textbf{\textup{(A)}}] It holds that $\mathbb{E}[\bm{y}_{t}|\bm{x}_{0}]=\phi_{t}(\bm{x}_{0})$ for any $t\in[0,T]$.
    
    \item[\textbf{\textup{(B)}}] For any $\bm{x}\in\mathbb{R}^{d}$, the function $\phi_{t}(\bm{x})$ in variable $t$ is right-continuous at $t=0$.
   
    \item[\textbf{\textup{(C)}}] For any $t,t'\in[0,T]$ satisfying $t\neq t'$, $\bm{y}_{t}$ and $\bm{y}_{t'}$ are independent if they are conditioned on $\bm{x}_{0}$.
\end{enumerate}
\end{assumption}
Note that the condition \textbf{\textup{(A)}} of Assumption~\ref{assumption:assumptions} extends assumptions considered in~\cite{lehtinen2018noise,huang2021neighbor} to a setting where we need to deal with the mapping $\phi_{t}$ together.
This condition means that the noisy image $\bm{y}_{t}$ observed at each time $t$ is denatured by some mapping $\phi_{t}$ and perturbed by some noise.
Note that if we only consider the case that $T=0$ and $\phi_{0}$ is the identity mapping, then this modeling can recover the conditions studied in~\cite{lehtinen2018noise,huang2021neighbor}.
Also note that the condition \textbf{\textup{(A)}} of Assumption~\ref{assumption:assumptions} includes the modeling considered in~\cite{xu2021deformed} as a special case.
Based on these conditions, similar to usual least-squares problems (see e.g.~\cite{hastie2009elements}), we can ensure that minimization for~\eqref{eq:time-aware denoising model} yields the desired result.

\begin{prop}
\label{thm:guarantee for population minimizer}
Suppose that all the conditions in Assumption~\ref{assumption:assumptions} hold.
Suppose also that a minimizer $f^{*}$ of the objective function~\eqref{eq:time-aware denoising model} over all measurable mapping from $\mathbb{R}^{d}\times [0,T]$ to $\mathbb{R}^{d}$ meets the condition that $f^{*}(\widetilde{\bm{y}}_{0},t)$ is right-continuous at $t=0$ for almost every $\widetilde{\bm{y}}_{0}\in\mathbb{R}^{d}$ with its probability distribution.
Then, the minimizer $f^{*}$ meeting the conditions above satisfies
\begin{align}
    \mathbb{E}_{\bm{y}_{0}}[f^{*}(\bm{y}_{0},0)|\bm{x}_{0}]=\bm{x}_{0}.
\end{align}
\end{prop}
Proposition~\ref{thm:guarantee for population minimizer} extends the theory on \emph{Noise2Noise}~\cite{lehtinen2018noise} shown in Section~3.1 of~\cite{zhussip2019extending} and Theorem~1 of \cite{huang2021neighbor} to a setting that observed data may contain denaturation.
The proof of Proposition~\ref{thm:guarantee for population minimizer} is deferred to~\ref{appsubsec:proof of proposition 1}.

\subsection{On Empirical Risk Minimizers}
\label{subsec:on empirical risk minimizers}
Once obtaining the theoretical justification for the population minimizer, it is also important to investigate how close the prediction $\widehat{\bm{x}}_{0}$ defined with the empirical risk minimizer and the true clean $\bm{x}_{0}$ are.
We address this problem by evaluating the Euclidean norm between them.
To present the result, let us introduce the empirical loss
\begin{equation}
\label{eq:empirical time-aware loss}
    \frac{1}{MN}\sum_{j=1}^{M}\sum_{i=1}^{N}\|f(\bm{y}_{0,j},\tau_{i})-\bm{y}_{\tau_{i},j}\|_{2}^{2}.
\end{equation}
Let $\widehat{f}$ be the Empirical Risk Minimizer (ERM) of~\eqref{eq:empirical time-aware loss} over a class of Lipschitz functions, where note that we select this class since the learning theory of Lipschitz spaces is well-investigated in~\cite{von2004distance}.
We define the predicted image as $\widehat{\bm{x}}_{0}=\frac{1}{M}\sum_{j=1}^{M}\widehat{f}(\bm{y}_{0,j},0)$.
Then, the following evaluation holds:
\begin{thm}
\label{thm:non-asymptotic theory}
Let $U= [0,1]^{d}$, $B_{1}=\sqrt{d}$, and $L>0$.
Let $\mathcal{F}=\{f:U\times[0,1]\to \mathbb{R}^{d}|f(\bm{0},0)=\bm{0}, f=(f_{1},\cdots,f_{d}), f_{1},\cdots,f_{d}\textup{ are }L\textup{-Lipschitz}\}$, where $f_{k}$ is $L$-Lipschitz in the sense that 
$
|f_{k}(\bm{u}_{1},t_{1})-f_{k}(\bm{u}_{2},t_{2})|\leq L\sqrt{\|\bm{u}_{1}-\bm{u}_{2}\|_{2}^{2}+|t_{1}-t_{2}|^{2}}
$ 
holds for any $\bm{u}_{1},\bm{u}_{2}\in U$ and $t_{1},t_{2}\in[0,1]$.
Suppose that there exists $\widehat{f}\in\mathcal{F}$ such that $\widehat{f}$ is the ERM~\eqref{eq:empirical time-aware loss} over $\mathcal{F}$.
Suppose also that there exists some $B_{2}\geq 0$ such that for every $i\in\{1,\cdots,N\}$ and $j\in\{1,\cdots,M\}$, $\|\bm{y}_{\tau_{i},j}\|_{2}\leq B_{2}$ holds almost surely.
We further assume that $\bm{y}_{\tau_{1},j},\cdots,\bm{y}_{\tau_{N},j}$ are i.i.d. for each $j\in\{1,\cdots,M\}$.
Let $\delta>Me^{-N}$, and $B=(L\sqrt{B_{1}^{2}+1}+B_{2})^{2}$.
Then, under Assumption~\ref{assumption:assumptions}, there exists a constant $C>0$ that is independent of $M,N,\delta$ such that with probability at least $1-4\delta$, it holds for the prediction $\widehat{\bm{x}}_{0}=\frac{1}{M}\sum_{j=1}^{M}\widehat{f}(\bm{y}_{0,j},0)$ that
\begin{align*}
    \|\widehat{\bm{x}}_{0}-\bm{x}_{0}\|_{2}^{2}&\leq 
    C(E_{\mathcal{F}}+G_{\phi}+L^{2}\mathbb{E}[\tau^{2}]+N^{-\frac{1}{2}}+\sqrt{(M \wedge N)^{-1}\log(M/\delta)}),
\end{align*}
where $E_{\mathcal{F}}=\inf_{f^{*}\in\mathcal{F}}\mathbb{E}[\|f^{*}(\bm{y}_{0},\tau)-\bm{y}_{\tau}\|_{2}^{2}]$, and $G_{\phi}=\|\mathbb{E}_{\tau}[\phi_{\tau}(\bm{x}_{0})]-\bm{x}_{0}\|_{2}^{2}$.
\end{thm}
The proof of Theorem~\ref{thm:non-asymptotic theory} basically relies on the notion called \emph{Rademacher complexity}~\cite{bartlett2002rademacher} and several concentration inequalities (see e.g.,~\cite{mohri2018foundations,vershynin2018high}).
The detailed proof of Theorem~\ref{thm:non-asymptotic theory} is deferred to~\ref{appsubsec:proof of theorem 1}.
To the best of our knowledge, the theoretical analysis dealing with the quantitative evaluation of the prediction made by the ERM is little investigated in the context of self-supervised denoising.
Here we describe several interpretations of this theorem.
The upper bound consists of four different kinds of quantities: the approximation error $E_{\mathcal{F}}$, the gap term $G_{\phi}$, the averaged squared-time $\mathbb{E}[\tau^{2}]$, and some constant terms that will vanish if $M,N\to\infty$.
The approximation error depends on the hardness of the denoising task and is expected to be small as long as $\mathcal{F}$ is a sufficiently large class.
The gap term measures to what extent the transformed image $\phi_{t}(\bm{x}_{0})$ is far apart from the clean $\bm{x}_{0}$ on average, which depends on the property of $\phi_{t}$ and $\bm{x}_{0}$.
The average time reveals some aspect of what kind of denatured dataset we use, since this term can become larger as the time for collecting images tend to be longer.
Therefore, Theorem~\ref{thm:non-asymptotic theory} implies that as long as both the denoising task and dataset are not too hard to deal with, training with the loss function~\eqref{eq:empirical time-aware loss} leads to successful results with high probability.

\section{Denatured-Noise2Noise}
\label{sec:proposed methodology}
The previous section reveals both the guarantee and limitation of the algorithm using~\eqref{eq:time-aware denoising model}.
We next move on to investigation of the empirical perspectives.
To this end, we employ the loss~\eqref{eq:time-aware denoising model} as the backbone of the self-supervised denoising framework we study.
In addition, we incorporate several techniques and regularization terms to enhance the performance.
Since the main loss function is extended from the existing loss called \emph{Noise2Noise}~\cite{lehtinen2018noise} to deal with denaturation in observed data, we call this framework \textit{Denatured-Noise2Noise} (DN2N).
We note that an illustration for the diagram can be found in~\eqref{subfig: prediction model of dn2n} of Figure~\ref{fig: overview of our method}.

\begin{figure*}[!t]
\centering
  \begin{minipage}[b]{.185\linewidth}
      \centering
        \includegraphics[width=.6\hsize]{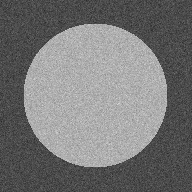}%
        \subcaption{$\bm{y}_0$}
        \label{subfig: training toy at 0}
  \end{minipage}%
  \begin{minipage}[b]{.185\linewidth}
      \centering
        \includegraphics[width=.6\hsize]{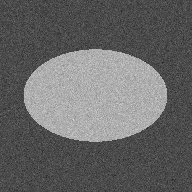}
        \subcaption{$\bm{y}_t$}
        \label{subfig: training toy at 14}
  \end{minipage}%
  \begin{minipage}[b]{.25\linewidth}
      \centering
        \includegraphics[width=\hsize]{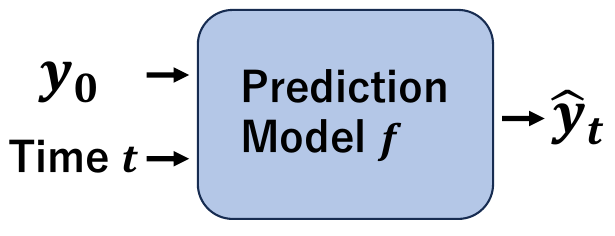}%
        \subcaption{Model}
        \label{subfig: prediction model of dn2n}
  \end{minipage}%
  \begin{minipage}[b]{.185\linewidth}
      \centering
        \includegraphics[width=.6\hsize]{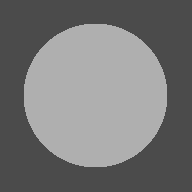}
        \subcaption{Clean}
        \label{subfig: clean toy at 0}
  \end{minipage}
  \begin{minipage}[b]{.185\linewidth}
      \centering
        \includegraphics[width=.6\hsize]{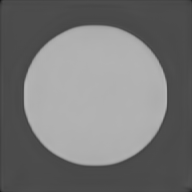}
        \subcaption{Prediction}
        \label{subfig: prediction toy dn2n}
  \end{minipage}
\caption{ 
    Overview of our method. 
    \eqref{subfig: training toy at 0} and~\eqref{subfig: training toy at 14}: 
    Examples of training images for our method, where $\bm{y}_t,\;t>0$ (resp. $\bm{y}_0$) denotes a denatured (resp. non-denatured) noisy image at time $\tau$ (resp. time $0$), and 
    the circle in~\eqref{subfig: training toy at 0} is degenerated into the ellipse in~\eqref{subfig: training toy at 14} as $t$ increases.
    \eqref{subfig: prediction model of dn2n}: 
    Our prediction model $f$ based on a deep neural network, 
    where in the prediction phase, the prediction is given by $\hat{\bm{y}}_{0}$.
    \eqref{subfig: clean toy at 0}: 
    Corresponding clean image $\bm{x}_0$ for $\bm{y}_0$. 
    \eqref{subfig: prediction toy dn2n}: 
    Prediction result $\hat{\bm{x}}_0$ by our method for $\bm{x}_0$.
    The detail of the setting for the simulation in \eqref{subfig: prediction toy dn2n} can be found in Section~\ref{subsubsec: expt1 setting}.
    }
\label{fig: overview of our method}   
\end{figure*}

\paragraph{Transforming Input Images}
In practice, we often deal with the case that both $M$ and $N$ are relatively small.
As seen in Theorem~\ref{thm:non-asymptotic theory}, in such situations the generalization performance of denoising models can be worsened.
To overcome this problem, we consider to transform data $\bm{y}_{0}$ to create multiple noisy images.
More precisely, let $\bm{m}$ be a random mapping on $\mathbb{R}^{d}$ such that for any $\bm{y}$, the output $\bm{m}(\bm{y})$ is defined as $\bm{m}(\bm{y})=\bm{y}+\bm{\varepsilon}$, where $\bm{y}\in\mathbb{R}^{d}$, and $\bm{\varepsilon}$ is a noise random variable that is independent of any other random variables and satisfies $\mathbb{E}[\bm{\varepsilon}|\bm{x}_{0}]=\bm{0}$.
Let $\bm{m}'$ be another mapping defined as in $\bm{m}$ with random noise $\bm{\varepsilon}'$ independently and identically distributed with $\bm{\varepsilon}$.
Using this mapping, we redefine the loss function as
\begin{align*}
\mathcal{L}_{D}(f)=\mathbb{E}_{\bm{y},\tau,\bm{m}}\left[\|f(\bm{m}(\bm{y}_{0}),\tau)-\bm{m}'(\bm{y}_{\tau})\|_{2}^{2}|\bm{x}_{0}\right].
\end{align*}
Note that this additional noise does not violate Assumption~\ref{assumption:assumptions} by the definition.
We remark that the approach that considers to add extra noise to the input noisy image has been utilized by several previous works~\cite{xu2020noisy,moran2020noisiser,pang2021recorrupted}.
Notably, we find that the existence of the additional noise also influences the performance of DN2N that takes denatured images as the inputs (see Section~\ref{subsec:implementation}).

In training, we minimize the empirical risk
\begin{equation}
    \label{eq:empirical risk}
    \widehat{\mathcal{L}}_{D}(f)=
    \frac{1}{LMN}\sum_{k=1}^{L}\sum_{j=1}^{M}\sum_{i=1}^{N}\|f(\bm{m}_{k}(\bm{y}_{0,j}),\tau_{i})-\bm{m}_{k}'(\bm{y}_{\tau_{i},j})\|_{2}^{2},
\end{equation}
where 
$\{\bm{m}_{k}\}_{k=1}^{L}$ and $\{\bm{m}_{k}'\}_{k=1}^{L}$ are random transforms defined with i.i.d. random noise $\{\bm{\varepsilon}_{k}\}_{k=1}^{L}$ and $\{\bm{\varepsilon}'_{k}\}_{k=1}^{L}$.

\paragraph{Averaging Loss}
In addition to the above problem, we also need to overcome the issue that the denaturation levels of data may deteriorate the denoising performance of~\eqref{eq:time-aware denoising model}.
Toward addressing the issue, we take one step into making use of more information contained in the denatured images.
Since the learning process is affected by both the denaturation and noise in the training images, we consider to alleviate only the noise and learn some structure of the denaturation.
Following a similar approach to the classical averaging technique commonly referred to as \emph{2D classification}~\cite{Campbell2012-rnCTFFIND,Zheng2017-rxMotionCor2} in Cryo-EM image analysis, we further enforce the denoising model to incorporate less-noisy image constructed by averaging a series of denatured noisy images $\{\bm{y}_{\tau_{i}}\}_{i=1}^{N}$.
We define the averaging loss as
\begin{align*}
    \mathcal{L}_{A}(f)=\mathbb{E}\left[
    \frac{1}{M}\sum_{j=1}^M
    \left\|\frac{1}{N}\sum_{i=1}^{N}f(\bm{m}(\bm{y}_{0,j}),\tau_{i})-\overline{\bm{y}}_{j}\right\|_{2}^{2}
    \right],
\end{align*}
where 
$\overline{\bm{y}}_{j}=\frac{1}{N}\sum_{i=1}^{N}\bm{m}'(\bm{y}_{\tau_{i},j})$,
and the expectation is taken for all the random variables.
Then, using $\bm{m}_{k}$ and $\bm{m}_{k}'$ in~\eqref{eq:empirical risk}, 
the empirical loss is also defined as
$$
    \widehat{\mathcal{L}}_{A}(f)=\frac{1}{LM}\sum_{k=1}^L\sum_{j=1}^{M}\left\|\frac{1}{N} \sum_{i=1}^{N}
    f(\bm{m}_{k}(\bm{y}_{0,j}),\tau_{i})    
    - \frac{1}{N}\sum_{i=1}^{N}\bm{m}_{k}'(\bm{y}_{\tau_{i},j}) \right\|_{2}^{2}.
$$

\paragraph{Training and Prediction}
The final version of the empirical loss in the denoising module is 
\begin{align*}
\widehat{\mathcal{L}}_{T}(f)=\widehat{\mathcal{L}}_{D}(f)+\frac{\mu}{LMN}\cdot\widehat{\mathcal{L}}_{A}(f),
\end{align*}
where $\mu\geq 0$ is a hyperparameter to be tuned.
Note that in the above definition, we further multiply the coefficient $(LMN)^{-1}$ to $\mu\widehat{\mathcal{L}}_{A}(f)$ to control the effect of the regularizer to the main loss $\widehat{\mathcal{L}}_{D}(f)$, depending on the number of noisy images available.
Let $\widehat{f}\in\textup{argmin}_{f\in\mathcal{F}}\widehat{\mathcal{L}}_{T}(f)$, where $\mathcal{F}$ is a space of denoising models.
After training with~\eqref{eq:empirical risk}, we predict the non-denatured clean image by
\begin{align}
\label{eq: prediction formula}
    \widehat{x}_0 = \frac{1}{KM} \sum_{i=1}^{K} \sum_{j=1}^{M} \widehat{f}(\bm{y}_{0,j}+\bm{\varepsilon}_{i},0),
\end{align}
where $K\in\mathbb{N}$.
Note that~\cite{quan2020self} also propose an empirical technique that averages outputs from a model to produce the prediction.
\cite{quan2020self} particularly utilize \emph{dropout}~\cite{srivastava2014dropout} for the inference.
Meanwhile, we consider additive random noise, different from~\cite{quan2020self}.

\section{Experiments}
\label{sec:experiments}
We conduct the following four numerical experiments to demonstrate the performance of DN2N: \textit{Expt.1} to \textit{Expt.4}. We describe the detail of each experiment from Section~\ref{subsec: expt1} to Section~\ref{subsec: expt4}. The purpose of the first experiment is to evaluate the consistency between our analytical result in Section~\ref{sec:theoretical analysis} and the prediction performance of our method on toy datasets, while investigating on what dataset our method potentially performs well. The second experiment is intended to quantitatively compare our method to existing representative methods for an MRI image dataset under a setting similar to~\cite{xu2021deformed}. The goal in the third (resp. fourth) experiment is to qualitatively evaluate the performance of our method by a Cryo-EM (resp. Fluorescence microscopy) image dataset, where no clean image is available for the evaluation. In Section~\ref{subsec:implementation}, we explain our implementation, the detail of the settings, the hyperparameter tuning, and ablation study. 

Throughout all the experiments except Expt.3, following~\cite{lehtinen2018noise,quan2020self,huang2021neighbor}, we evaluate the prediction performance of our method on the noisy dataset by two standard metrics in vision domain: Peak Signal to Noise Ratio (PSNR) (see e.g., \cite{hore2010image}) and Structural Similarity Index Measure (SSIM)~\cite{1284395}.
Note that we use the package called \texttt{scikit-image}~\cite{vanderwalt2014image} for implementation of these metrics. In the experiments, we use four NVIDIA V100 GPUs combined with two Intel Xeon Gold 6148 processors.

\subsection{Expt.1: Evaluation of Theoretical Results using Toy Datasets}
\label{subsec: expt1}
\subsubsection{Setting}
\label{subsubsec: expt1 setting}
We firstly craft two clean image datasets. All the images in the two datasets are gray-scale with the size 192$\times$192, and the number of images in each dataset is $25$.
The two datasets share the same clean image at time $0$ (see the clean image in~\eqref{subfig: clean toy at 0} of Figure~\ref{fig: overview of our method}), while the denaturing speed is different. Let $\mathcal{D}_1^{\rm toy}$ and $\mathcal{D}_2^{\rm toy}$ denote the two clean datasets. In $\mathcal{D}_1^{\rm toy}$, a circle in~\eqref{subfig: clean toy at 0} is denatured \emph{slowly} into an ellipse as the time index increases, whereas in $\mathcal{D}_2^{\rm toy}$, the denaturing speed is \emph{faster}; the formal definitions of $\mathcal{D}_1^{\rm toy}$ and $\mathcal{D}_2^{\rm toy}$ will be presented in the subsequent paragraph below. 
Following~\cite{khademi2021self}, we then add Poisson noise 
with the hyperparameter $\lambda>0$
and Gaussian noise 
with the standard deviation $\sigma$
for both the clean datasets; the formal definitions of the noises are also deferred to the paragraph below.
Two example images of $\mathcal{D}_2^{\rm toy}$ with Poisson-Gaussian noises are displayed in~\eqref{subfig: training toy at 0} and~\eqref{subfig: training toy at 14} of Figure~\ref{fig: overview of our method}.

We then describe the formal definitions of the toy dataset and noises in the following paragraphs separately.

\begin{figure*}[!t]
\centering
    \begin{tabular}{ccccc}
    \centering
      \begin{minipage}[b]{.175\linewidth}
      \centering
        \includegraphics[width=.6\hsize]{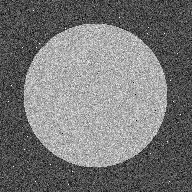}
        \subcaption{}
        \label{sub-fig: toy-slow-t0-25-25}
      \end{minipage} &
      \begin{minipage}[b]{.175\linewidth}
      \centering
        \includegraphics[width=.6\hsize]{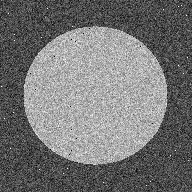}%
        \subcaption{}
        \label{sub-fig: toy-slow-t6-25-25}
      \end{minipage} &
      \begin{minipage}[b]{.175\linewidth}
      \centering
        \includegraphics[width=.6\hsize]{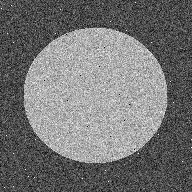}%
        \subcaption{}
        \label{sub-fig: toy-slow-t12-25-25}
      \end{minipage} &
      \begin{minipage}[b]{.175\linewidth}
      \centering
        \includegraphics[width=.6\hsize]{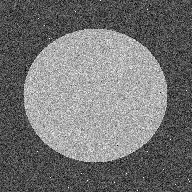}%
        \subcaption{}
        \label{sub-fig: toy-slow-t18-25-25}
      \end{minipage} &
      \begin{minipage}[b]{.175\linewidth}
      \centering
        \includegraphics[width=.6\hsize]{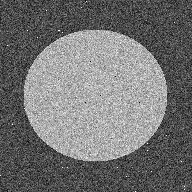}%
        \subcaption{}
        \label{sub-fig: toy-slow-t24-25-25}
      \end{minipage}\\

      \begin{minipage}[b]{.175\linewidth}
      \centering
        \includegraphics[width=.6\hsize]{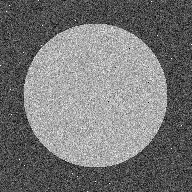}
        \subcaption{}
        \label{sub-fig: toy-fast-t0-25-25}
      \end{minipage} &
      \begin{minipage}[b]{.175\linewidth}
      \centering
        \includegraphics[width=.6\hsize]{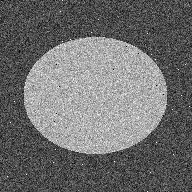}%
        \subcaption{}
        \label{sub-fig: toy-fast-t6-25-25}
      \end{minipage} &
      \begin{minipage}[b]{.175\linewidth}
      \centering
        \includegraphics[width=.6\hsize]{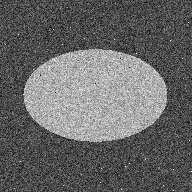}%
        \subcaption{}
        \label{sub-fig: toy-fast-t12-25-25}
      \end{minipage} &
      \begin{minipage}[b]{.175\linewidth}
      \centering
        \includegraphics[width=.6\hsize]{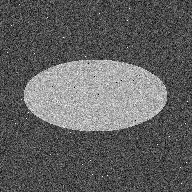}%
        \subcaption{}
        \label{sub-fig: toy-fast-t18-25-25}
      \end{minipage} &
      \begin{minipage}[b]{.175\linewidth}
      \centering
        \includegraphics[width=.6\hsize]{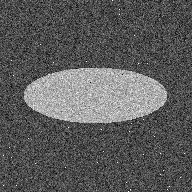}%
        \subcaption{}
        \label{sub-fig: toy-fast-t24-25-25}
      \end{minipage}\\

      \begin{minipage}[b]{.175\linewidth}
      \centering
        \includegraphics[width=.6\hsize]{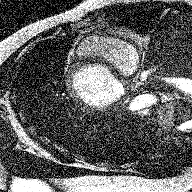}
        \subcaption{}
        \label{sub-fig: mri-noisy-id2-t0}
      \end{minipage} &
      \begin{minipage}[b]{.175\linewidth}
      \centering
        \includegraphics[width=.6\hsize]{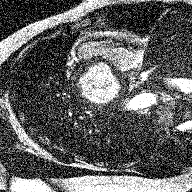}%
        \subcaption{}
        \label{sub-fig: mri-noisy-id2-t6}
      \end{minipage} &
      \begin{minipage}[b]{.175\linewidth}
      \centering
        \includegraphics[width=.6\hsize]{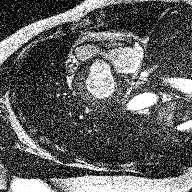}%
        \subcaption{}
        \label{sub-fig: mri-noisy-id2-t12}
      \end{minipage} &
      \begin{minipage}[b]{.175\linewidth}
      \centering
        \includegraphics[width=.6\hsize]{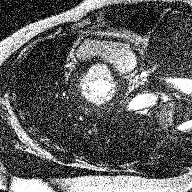}%
        \subcaption{}
        \label{sub-fig: mri-noisy-id2-t18}
      \end{minipage} &
      \begin{minipage}[b]{.175\linewidth}
      \centering
        \includegraphics[width=.6\hsize]{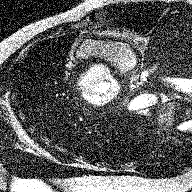}%
        \subcaption{}
        \label{sub-fig: mri-noisy-id2-t24}
      \end{minipage}\\

      \begin{minipage}[b]{.175\linewidth}
      \centering
        \includegraphics[width=.6\hsize]{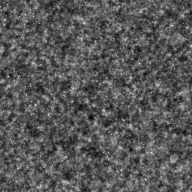}
        \subcaption{}
        \label{sub-fig: cryoem-original-top-left-t1}
      \end{minipage} &
      \begin{minipage}[b]{.175\linewidth}
      \centering
        \includegraphics[width=.6\hsize]{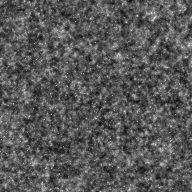}%
        \subcaption{}
        \label{sub-fig: cryoem-original-top-left-t4}
      \end{minipage} &
      \begin{minipage}[b]{.175\linewidth}
      \centering
        \includegraphics[width=.6\hsize]{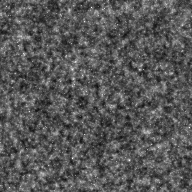}%
        \subcaption{}
        \label{sub-fig: cryoem-original-top-left-t8}
      \end{minipage} &
      \begin{minipage}[b]{.175\linewidth}
      \centering
        \includegraphics[width=.6\hsize]{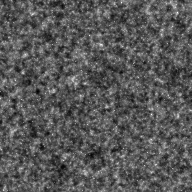}%
        \subcaption{}
        \label{sub-fig: cryoem-original-top-left-t12}
      \end{minipage} &
      \begin{minipage}[b]{.175\linewidth}
      \centering
        \includegraphics[width=.6\hsize]{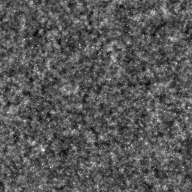}%
        \subcaption{}
        \label{sub-fig: cryoem-original-top-left-t15}
      \end{minipage}\\

      \begin{minipage}[b]{.175\linewidth}
      \centering
        \includegraphics[width=.6\hsize]{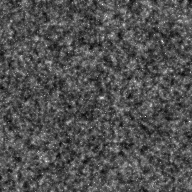}
        \subcaption{}
        \label{sub-fig: cryoem-original-bottom-right-t1}
      \end{minipage} &
      \begin{minipage}[b]{.175\linewidth}
      \centering
        \includegraphics[width=.6\hsize]{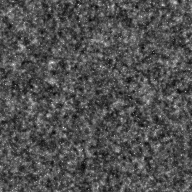}%
        \subcaption{}
        \label{sub-fig: cryoem-original-bottom-right-t4}
      \end{minipage} &
      \begin{minipage}[b]{.175\linewidth}
      \centering
        \includegraphics[width=.6\hsize]{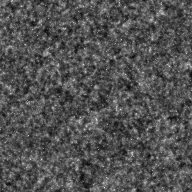}%
        \subcaption{}
        \label{sub-fig: cryoem-original-bottom-right-t8}
      \end{minipage} &
      \begin{minipage}[b]{.175\linewidth}
      \centering
        \includegraphics[width=.6\hsize]{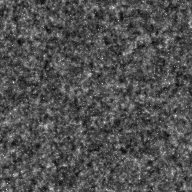}%
        \subcaption{}
        \label{sub-fig: cryoem-original-bottom-right-t12}
      \end{minipage} &
      \begin{minipage}[b]{.175\linewidth}
      \centering
        \includegraphics[width=.6\hsize]{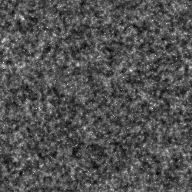}%
        \subcaption{}
        \label{sub-fig: cryoem-original-bottom-right-t15}
      \end{minipage}\\
      
      \begin{minipage}[b]{.175\linewidth}
      \centering
        \includegraphics[width=.6\hsize]{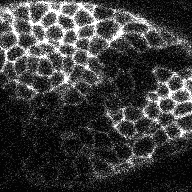}
        \subcaption{}
        \label{sub-fig: mfm-noisy-t0}
      \end{minipage} &
      \begin{minipage}[b]{.175\linewidth}
      \centering
        \includegraphics[width=.6\hsize]{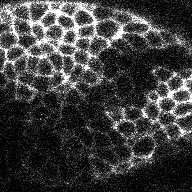}%
        \subcaption{}
        \label{sub-fig: mfm-noisy-t13}
      \end{minipage} &
      \begin{minipage}[b]{.175\linewidth}
      \centering
        \includegraphics[width=.6\hsize]{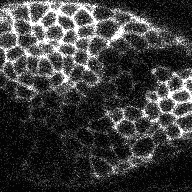}%
        \subcaption{}
        \label{sub-fig: mfm-noisy-t25}
      \end{minipage} &
      \begin{minipage}[b]{.175\linewidth}
      \centering
        \includegraphics[width=.6\hsize]{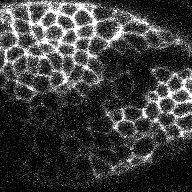}%
        \subcaption{}
        \label{sub-fig: mfm-noisy-t38}
      \end{minipage} &
      \begin{minipage}[b]{.175\linewidth}
      \centering
        \includegraphics[width=.6\hsize]{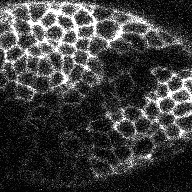}%
        \subcaption{}
        \label{sub-fig: mfm-noisy-t49}
      \end{minipage}
    \end{tabular}
     \caption{
         Visualization of noisy denaturation images used in Expt.1 to Expt.4. 
         The first \& second, the third, the fourth \& fifth, and the sixth rows are the toy images of Expt.1, MRI images~\cite{bernard2018deep} of Expt.2, Cryo-EM images~\cite{Iudin2023-ua} of Expt.3, and FM images~\cite{haward2020n2FMdataset} of Expt.4, respectively. The first column visualizes noisy images at time $0$ except for the fourth to sixth rows, and the time index increases from the first to fifth column; the fourth to sixth rows in the first column visualize noisy images at time $1$. The noisy images at time $0$ for the fourth, fifth, and sixth rows are visualized at \eqref{subfig: original-top-left}, ~\eqref{subfig: original-bottom-right}, and~\eqref{subfig: mfm-clipped-noisy-t0}, respectively.  
     }
     \label{fig: training dataset for each expt}
  \end{figure*}

\paragraph{Toy Dataset}
We introduce the definitions of two clean datasets: $\mathcal{D}_1^{\rm toy}=\left\{\bm{x}_0, \bm{x}_1^{(1)},\cdots,\bm{x}_i^{(1)},\cdots,\bm{x}_N^{(1)}\right\}$ and $\mathcal{D}_2^{\rm toy}=\left\{\bm{x}_0, \bm{x}_1^{(2)},\cdots,\bm{x}_i^{(2)},\cdots,\bm{x}_N^{(2)}\right\}$, where $N=24$, $\bm{x}_0$ is non-denatured clean image, and $\bm{x}_i^{(s)}$ $(i \in \mathbb{N})$ denotes denatured clean image with time index $i$ in $\mathcal{D}_s^{\rm toy}$. The clean image $\bm{x}_0$ is visualized in~\eqref{subfig: clean toy at 0} of Figure~\ref{fig: overview of our method}. Let us define the horizontal (resp. vertical) axis as $a$ (resp. $b$) axis to the clean image. In~\eqref{subfig: clean toy at 0}, the coordinates at the bottom-left, bottom-right, top-right, and top-left pixel are defined by $(1,192)$, $(192,192)$, $(192,1)$, and $(1,1)$, respectively. In addition, let $(97,97)$ denote the coordinates of the \emph{center} pixel in $\bm{x}_0$.
The gray circled region is defined by
\begin{equation}
\label{eq: denatured image at time t}
    \mathcal{R}_i := \left\{
    (a,b)\in \mathbb{N}^2\;\middle|\; 
    \frac{(a-97)^2}{q_i^2} + \frac{(b-97)^2}{1^2} \leq \left(\frac{3}{4} \cdot \frac{\ell}{2}\right)^2\right\}, 
\end{equation}
where $\ell=192$ and $q_0 = 1$. 
The remaining region in $\bm{x}_0$ is defined by $\mathcal{R}_{0}^{c}$. In Expt.1, we design two toy clean image datasets, in both of which the gray circle in $\bm{x}_0$ becomes an gray ellipse as the time index $i$ increases. The difference is the denaturing speed, and the ellipse region in $\bm{x}_i\;(i = 1,...,N)$ is defined by $\mathcal{R}_i$ of~\eqref{eq: denatured image at time t}, whose
\begin{equation}
\label{eq: definition of denaturing speed}
    q_i = \left\{
            \begin{array}{ll}
            \log\left(\sqrt{i}\cdot(\sqrt{e}-e)/N + e\right), & \textrm{if}\;\bm{x}_i \in \mathcal{D}_1^{\rm toy},
            \textrm{i.e.,\;``\textit{slow}''\;case},\\
            \exp\left(- i^{1.1}\log2 / N\right), & \textrm{if}\;\bm{x}_i \in \mathcal{D}_2^{\rm toy},
            \textrm{i.e.,\;``\textit{fast}''\;case},
            \end{array}
            \right.
\end{equation}
whereas the remaining region is defined by $\mathcal{R}_i^c$.
Note that $q_i = 1$ in both cases of~\eqref{eq: definition of denaturing speed} if $i=0$. 
Furthermore, $(a,b)$-th pixel value in $\bm{x}_i$, denoted by $x_i (a,b)$, is defined as follows:
\begin{equation*}
    x_i (a,b) = \left\{
            \begin{array}{ll}
            175, & \textrm{if\;}(a,b) \in \mathcal{R}_i,\\
            75, & \textrm{if\;}(a,b) \in \mathcal{R}_i^c.
            \end{array}
            \right.
\end{equation*}

\paragraph{Poisson-Gaussian Noise}
We follow~\cite{khademi2021self} to define a noisy image $\bm{y}$ with the Poisson-Gaussian noise as follows: 
$\bm{y}=\lambda^{-1}\bm{z}+\bm{\delta}$, where $\bm{\delta} \sim \mathcal{N}\left(\bm{0},\sigma^2 \bm{I}\right)$, $\bm{z}\sim \mathcal{P}(\lambda\cdot \bm{x})$, $\mathcal{P}(\beta)$ is the Poisson distribution with parameter $\beta$, $\lambda$ is a positive hyperparameter, and $\bm{x}$ is a clean image. Note that \cite{khademi2021self} further approximate the formulation to a Gaussian noise model. Meanwhile, we use the modeling \citet{khademi2021self} initially consider.

\paragraph{Definition of $\bm{y}_{\tau_i, j}$ in~\eqref{eq:empirical risk}}
Let $\Tilde{\mathcal{D}}_{s}^{\text{toy}}=(\bm{y}^{(s)}_0,\bm{y}^{(s)}_1,...,\bm{y}^{(s)}_{N})$ denote a sequence of noisy images with denaturation for training DN2N, where each $\bm{y}^{(s)}_{i}$ is constructed from $\bm{x}^{(s)}_{i} \in \mathcal{D}_s^{\text{toy}}$ with the above Poisson-Gaussian noise. 
Then, $\bm{y}_{\tau_i, j}$ is defined by the following two steps. Firstly, $\tau_i$ is taken from a set obtained by applying uniform permutation to $\left\{\frac{n}{10} \middle| n =1,...,N\right\}$, and secondly, $\bm{y}_{\tau_i, j}$ is defined by $\bm{y}_{10 \times \tau_i} \in \Tilde{\mathcal{D}}_{s}^{\text{toy}}$ (we slightly abuse the notation of $\bm{y}_{\tau_i}$ for simplicity). We note that the quantity $M$ of Section~\ref{sec:proposed methodology} is one in this experiment.


\paragraph{Supplementary Information for Toy Images in Figure~\ref{fig: overview of our method} and~\ref{fig: training dataset for each expt}}
For~\eqref{subfig: training toy at 0} (resp.~\eqref{subfig: training toy at 14}), we visualize $\bm{y}^{(2)}_0$ (resp. $\bm{y}^{(2)}_{12}$) with the hyperparameters $(\lambda, \sigma)=(10,10)$. 
In~\eqref{subfig: prediction toy dn2n}, the predicted image from our trained model
is shown, where the fast denaturation dataset $(\bm{y}^{(2)}_0,...,\bm{y}^{(2)}_{24})$ with $(\lambda,\sigma)=(10,10)$ are used.
At~\eqref{sub-fig: toy-slow-t0-25-25} to~\eqref{sub-fig: toy-fast-t24-25-25} in Figure~\ref{fig: training dataset for each expt}, we visualize how the dataset $\Tilde{\mathcal{D}}_{s}^{\text{toy}}$ denatures. In the figure, the left to right image in the first (resp. second) row correspond to $\bm{y}^{(1)}_{0}, \bm{y}^{(1)}_{6}, \bm{y}^{(1)}_{12}, \bm{y}^{(1)}_{18}, \bm{y}^{(1)}_{24}$ (resp. $\bm{y}^{(2)}_{0}, \bm{y}^{(2)}_{6}, \bm{y}^{(2)}_{12}, \bm{y}^{(2)}_{18}, \bm{y}^{(2)}_{24}$)
with the hyperparameters $(\lambda, \sigma)=(25,25)$.

\paragraph{Comparative Methods}
We employ three comparative methods to investigate DN2N: 
\begin{itemize}
    \item BM3D~\cite{dabov2007image} is a single image denoising method without deep learning. We apply the module~\cite{makinen2022bm3d} provided by the authors of~\cite{makinen2020collaborate} to the non-denatured noisy image $\bm{y}^{(s)}_{0}$.

    \item Topaz-Denoise (TD)~\cite{bepler2020topaz} is a variant of Noise2Noise (N2N)~\cite{lehtinen2018noise}. We denoise $\bm{y}^{(s)}_{0}$ using $\Tilde{\mathcal{D}}^{\mathrm{toy}}_s$ based on the implementation inspired by~\cite{bepler2020topaz}; see details in~\ref{append: further expt}.

    \item Noise2Fast (N2F)~\cite{lequyer2022noise} is a self-supervised single image denoising method. We run the GitHub code of~\cite{jasonlequyer2021n2f} using only $\bm{y}^{(s)}_{0}$.

\end{itemize}

\subsubsection{Result and Discussion}
\label{subsubsec: result and discussion on toy dataset}
The results are shown in Table~\ref{tab:toydata results}. For all the pairs $(\lambda, \sigma)$ in the table, DN2N works better for the datasets with slow denaturation, compared to the fast denaturation datasets. This tendency positively supports the theoretical finding in Theorem~\ref{thm:non-asymptotic theory}.

\begin{table*}[!t]
\caption{
Results of Expt.1: evaluation of predicting performance by three existing methods and DN2N on toy image datasets with different denaturing speed and noise intensity. The word “Slow” or “Fast” represents the denaturing speed. 
The higher number means the better performance in both PSNR and SSIM metrics.
}
\label{tab:toydata results}
    \centering
    \scalebox{0.8}{
        \begin{tabular}{cccccccccc}
        \toprule
\multirow{2}{*}{\rotatebox{90}{}}      & \multirow{2}{*}{$(\lambda,\sigma)$}   & \multicolumn{2}{c}{BM3D}  & \multicolumn{2}{c}{TD}  & \multicolumn{2}{c}{N2F} & \multicolumn{2}{c}{DN2N}  \\ \cline{3-10}
                                       &                      & PSNR$\uparrow$ & SSIM$\uparrow$ & PSNR & SSIM & PSNR & SSIM & PSNR & SSIM \\ \hline
\multirow{4}{*}{\rotatebox{90}{Slow}}  & $(25,10)$ & 37.14 & 0.989 & 33.49 & 0.965 & 38.72 & 0.990 & 31.64 & 0.968
  \\
                                       & $(10,10)$ & 36.91 & 0.989 & 34.32 & 0.970 & 38.98 & 0.991 & 30.97 & 0.966
 \\
                                       & $(25,25)$ & 36.04 & 0.985 & 34.38 & 0.947 & 37.47 & 0.985 & 30.85 & 0.963
 \\
                                       & $(10,25)$ & 35.86 & 0.985 & 34.35 & 0.954 & 36.79 & 0.982 & 31.75 & 0.966
 \\ \hline
\multirow{4}{*}{\rotatebox{90}{Fast}}  & $(25,10)$ & 36.81 & 0.988 & 29.15 & 0.940 & 38.83 & 0.990 & 29.40 & 0.956 \\
                                       & $(10,10)$ & 37.02 & 0.989 & 27.91 & 0.934 & 38.74 & 0.991 & 30.00 & 0.960 \\
                                       & $(25,25)$ & 36.08 & 0.986 & 28.51 & 0.927 & 37.19 & 0.984 & 27.77 & 0.940 \\
                                       & $(10,25)$ & 35.65 & 0.984 & 28.00 & 0.912 & 36.62 & 0.981 & 26.92 & 0.931 \\
            \bottomrule
        \end{tabular}
    }
\end{table*}

For TD, it performs better on the slow denaturation datasets than on the fast counterparts. This implies that the slow denaturation images $\bm{y}^{(1)}_1$, $\bm{y}^{(1)}_2$,...,$\bm{y}^{(1)}_{24}$ are more beneficial to accurately predict $\bm{x}_0$ than the fast denaturation images $\bm{y}^{(2)}_1,\bm{y}^{(2)}_2,...,\bm{y}^{(2)}_{24}$ for TD. Moreover, we can see that the performance of TD is more robust than DN2N against the fast denaturation datasets. We think the reason is as follows. During training of TD, the reconstructed image from $\bm{y}^{(2)}_0$ targets only the image $\bm{y}^{(2)}_1$, therefore TD learns how to predict $\bm{x}_0$ mainly from $\bm{y}^{(2)}_0$. On the other hand, during training of DN2N, the reconstructed image from $\bm{y}^{(2)}_0$ targets the various images of $\bm{y}^{(2)}_1$, $\bm{y}^{(2)}_2$,...,$\bm{y}^{(2)}_{24}$.
Since many images in $\bm{y}^{(2)}_1$,...,$\bm{y}^{(2)}_{24}$ are dissimilar with $\bm{y}^{(2)}_0$ (see e.g.,~\eqref{subfig: training toy at 14}, and~\eqref{sub-fig: toy-fast-t6-25-25} to~\eqref{sub-fig: toy-fast-t24-25-25}), 
the prediction performance of DN2N is more influenced by the dissimilar images than TD.

It is interesting that both BM3D and N2F outperform DN2N. This observation may be due to the dissimilarity between $\bm{y}^{(s)}_0$ and $\bm{y}^{(s)}_t$, as $t$ takes a larger value.
Additionally, the relatively simple $\bm{y}_0$ likely enhances the gap; see e.g.,~\eqref{subfig: training toy at 0},~\eqref{sub-fig: toy-slow-t0-25-25}, and~\eqref{sub-fig: toy-fast-t0-25-25}. Those results suggest that, even when the denaturation is weaker, if $\bm{y}_0$ is relatively simple, the assistance from $\bm{y}_1$, $\bm{y}_2$, ..., $\bm{y}_{N}$ to denoise $\bm{y}_0$ is not efficient enough for DN2N.

In Expt.2, we evaluate DN2N using a noisy denaturation MRI image dataset~\cite{bernard2018deep}, satisfying  (i) $\bm{y}_0$ is relatively complicated, and (ii) the denaturation is weak. The MRI images are visualized in the third row of Figure~\ref{fig: training dataset for each expt}, where $\bm{y}_0$ of the MRI images is shown in~\eqref{sub-fig: mri-noisy-id2-t0}.

\subsection{Expt.2: Performance Evaluation using MRI Image Dataset}
\label{subsec: expt2}
\subsubsection{Setting}
\label{subsubsec: setting expt2}
Following~\cite{xu2021deformed}, we employ time series MRI image dataset called ACDC~\cite{bernard2018deep}.
According to~\cite{bernard2018deep}, all the dataset contains gray-scaled images whose sizes are 192$\times$192 each. Among 20 healthy patients in the dataset, we firstly pick the first five IDs (i.e., from ID~1 to ID~5), and for each ID, we use the first 25 images in ascending order of time. Following~\cite{khademi2021self}, we then add the Poisson-Gaussian noise with the hyperparameter values $(\lambda, \sigma)=(10,50)$ to all the images independently; see an example for the noisy images in~\eqref{sub-fig: noisy-id1-t0} of Figure~\ref{fig: append acdc before after}. 

Let $\Tilde{\mathcal{D}} =(\bm{y}_{0},...,\bm{y}_{i},...,\bm{y}_{24})$ denote a sequence of the noisy denaturation images for training DN2N, where $i \in \{0\}\cup \mathbb{N}$ is the time index. As for definition of $\bm{y}_{\tau_i, j}$ in~\eqref{eq:empirical risk}, we follow the same procedure described in the third from the last paragraph of Section~\ref{subsubsec: expt1 setting}. Regarding the performance comparison, 
in addition to three methods described in the last paragraph of Section~\ref{subsubsec: expt1 setting}, we employ one more comparative method:
\begin{itemize}
    \item Deformed2Self (D2S)~\cite{xu2021deformed} is a self-supervised deformed image denoising method, which is the \emph{state-of-the-art} method for the MRI image dataset. We apply the GitHub code of~\cite{daviddmc2021d2s} to $\Tilde{\mathcal{D}}$.
\end{itemize}

\paragraph{Supplementary Information for the Third Row in Figure~\ref{fig: training dataset for each expt}}
The left to right image in the third row correspond to $\bm{y}_{0}$, $\bm{y}_{6}$, $\bm{y}_{12}$, $\bm{y}_{18}$, and $\bm{y}_{24}$ with ID~2.

\subsubsection{Result and Discussion}

\begin{table*}[!t]
\caption{
    Results of Expt.2: performance comparison using five persons' time series MRI images from ACDC. The best performance is in bold font, with the second best result in underlined. 
}
\label{tab:acdc results}
    \centering
    \scalebox{0.8}{
        \begin{tabular}{ccccccccccc}
        \toprule
            \multirow{2}{*}{Method} & \multicolumn{2}{c}{ID 1}  & \multicolumn{2}{c}{ID 2}  & \multicolumn{2}{c}{ID 3}  & \multicolumn{2}{c}{ID 4}  & \multicolumn{2}{c}{ID 5} \\ \cline{2-11}
             & PSNR$\uparrow$ & SSIM$\uparrow$ & PSNR & SSIM & PSNR & SSIM & PSNR & SSIM & PSNR & SSIM \\ \hline
            BM3D& 22.29 & 0.743 & 23.32 & 0.722 & 23.35 & 0.752 & 18.12 & 0.687 & 23.48 & 0.741 \\ 
            TD & 22.91 & 0.740 & 23.02 & 0.672 & 24.16 & 0.783 & \underline{21.23} & 0.769 & \underline{24.74} & \underline{0.776} \\ 
            N2F & 22.93 & 0.754 & 23.64 & 0.749 & 23.80 & 0.749 & 20.31 & 0.720 & 23.63 & 0.718 \\ 
            D2S & \textbf{23.92} & \underline{0.807} & \underline{24.80} & \underline{0.752} & \underline{25.04} & \underline{0.784} & 21.11 & \underline{0.778} & \textbf{24.83} & \textbf{0.777} \\ 
            DN2N & \underline{23.16} & \textbf{0.808} & \textbf{24.85} & \textbf{0.790} & \textbf{25.60} & \textbf{0.843} & \textbf{21.24} & \textbf{0.782} & 22.79 & 0.763 \\ 
            \bottomrule
        \end{tabular}
    }
\end{table*}

\begin{figure*}[!t]
\centering
    \begin{tabular}{ccccc}
    \centering
      \begin{minipage}[b]{.175\linewidth}
      \centering
        \includegraphics[width=.65\hsize]{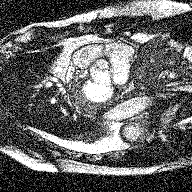}
        \subcaption{}
        \label{sub-fig: noisy-id1-t0}
      \end{minipage} &
      \begin{minipage}[b]{.175\linewidth}
      \centering
        \includegraphics[width=.65\hsize]{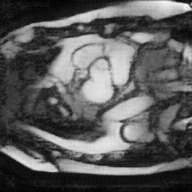}%
        \subcaption{}
        \label{sub-fig: td-pred-id1}
      \end{minipage} &
      \begin{minipage}[b]{.175\linewidth}
      \centering
        \includegraphics[width=.65\hsize]{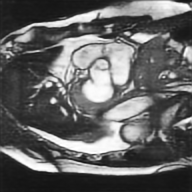}%
        \subcaption{}
        \label{sub-fig: d2s-pred-id1}
      \end{minipage} &
      \begin{minipage}[b]{.175\linewidth}
      \centering
        \includegraphics[width=.65\hsize]{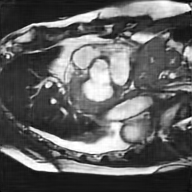}%
        \subcaption{}
        \label{sub-fig: dn2n-pred-id1}
      \end{minipage} &
      \begin{minipage}[b]{.175\linewidth}
      \centering
        \includegraphics[width=.65\hsize]{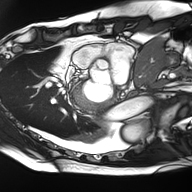}%
        \subcaption{}
        \label{sub-fig: clean-id1-t0}
      \end{minipage} \\

      \begin{minipage}[b]{.175\linewidth}
      \centering
        \includegraphics[width=.65\hsize]{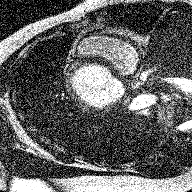}
        \subcaption{}
        \label{sub-fig: noisy-id2-t1}
      \end{minipage} &
      \begin{minipage}[b]{.175\linewidth}
      \centering
        \includegraphics[width=.65\hsize]{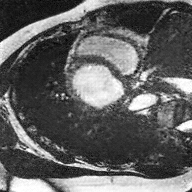}%
        \subcaption{}
        \label{sub-fig: td-pred-id2}
      \end{minipage} &
      \begin{minipage}[b]{.175\linewidth}
      \centering
        \includegraphics[width=.65\hsize]{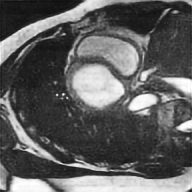}
        \subcaption{}
        \label{sub-fig: d2s-pred-id2}
      \end{minipage} &
      \begin{minipage}[b]{.175\linewidth}
      \centering
        \includegraphics[width=.65\hsize]{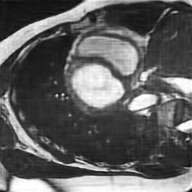}
        \subcaption{}
        \label{sub-fig: dn2n-pred-id2}
      \end{minipage} &
      \begin{minipage}[b]{.175\linewidth}
      \centering
        \includegraphics[width=.65\hsize]{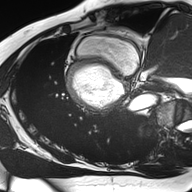}
        \subcaption{}
        \label{sub-fig: clean-id2-t0}
      \end{minipage} \\

      \begin{minipage}[b]{.175\linewidth}
            \centering
            \includegraphics[width=.65\hsize]{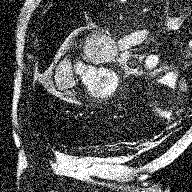}%
            \subcaption{}
            \label{subfig: noisy-id3-t0}
        \end{minipage} &
        \begin{minipage}[b]{.175\linewidth}
          \centering
            \includegraphics[width=.65\hsize]{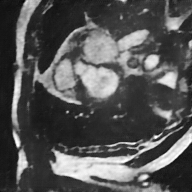}%
            \subcaption{}
            \label{sub-fig: td-pred-id3}
      \end{minipage} &
        \begin{minipage}[b]{.175\linewidth}
            \centering
            \includegraphics[width=.65\hsize]{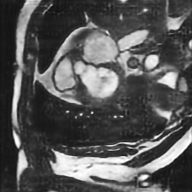}%
            \subcaption{}
            \label{subfig: d2s-pred-id3}
        \end{minipage} &
        \begin{minipage}[b]{.175\linewidth}
            \centering
            \includegraphics[width=.65\hsize]{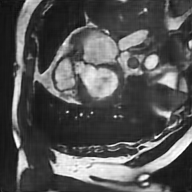}%
            \subcaption{}
            \label{subfig: dn2n-pred-id3}
        \end{minipage} &
        \begin{minipage}[b]{.175\linewidth}
            \centering
            \includegraphics[width=.65\hsize]{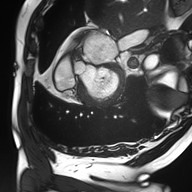}%
            \subcaption{}
            \label{subfig: clean-id3-t0}
        \end{minipage} \\

      \begin{minipage}[b]{.175\linewidth}
      \centering
        \includegraphics[width=.65\hsize]{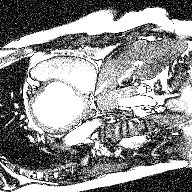}%
        \subcaption{}
        \label{sub-fig: noisy-id4-t0}
      \end{minipage} &
      \begin{minipage}[b]{.175\linewidth}
      \centering
        \includegraphics[width=.65\hsize]{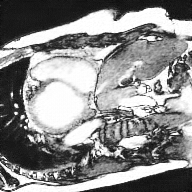}%
        \subcaption{}
        \label{sub-fig: td-pred-id4}
      \end{minipage} &
      \begin{minipage}[b]{.175\linewidth}
      \centering
        \includegraphics[width=.65\hsize]{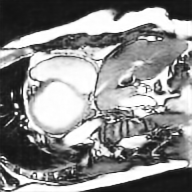}%
        \subcaption{}
        \label{sub-fig: d2s-pred-id4}
      \end{minipage} &
      \begin{minipage}[b]{.175\linewidth}
      \centering
        \includegraphics[width=.65\hsize]{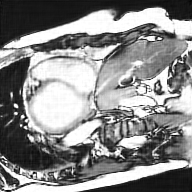}%
        \subcaption{}
        \label{sub-fig: dn2n-pred-id4}
      \end{minipage} &
      \begin{minipage}[b]{.175\linewidth}
      \centering
        \includegraphics[width=.65\hsize]{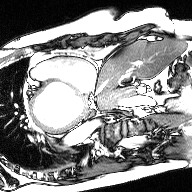}%
        \subcaption{}
        \label{sub-fig: clean-id4-t0}
      \end{minipage} \\

      \begin{minipage}[b]{.175\linewidth}
      \centering
        \includegraphics[width=.65\hsize]{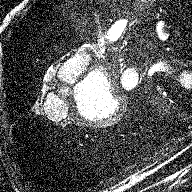}%
        \subcaption{}
        \label{sub-fig: noisy-id5-t0}
      \end{minipage} &
      \begin{minipage}[b]{.175\linewidth}
          \centering
            \includegraphics[width=.65\hsize]{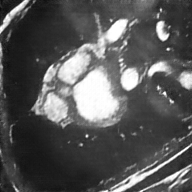}%
            \subcaption{}
            \label{sub-fig: td-pred-id5}
      \end{minipage} &
      \begin{minipage}[b]{.175\linewidth}
      \centering
        \includegraphics[width=.65\hsize]{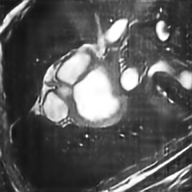}%
        \subcaption{}
        \label{sub-fig: d2s-pred-id5}
      \end{minipage} &
      \begin{minipage}[b]{.175\linewidth}
      \centering
        \includegraphics[width=.65\hsize]{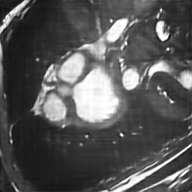}%
        \subcaption{}
        \label{sub-fig: dn2n-pred-id5}
      \end{minipage} &
      \begin{minipage}[b]{.175\linewidth}
      \centering
        \includegraphics[width=.65\hsize]{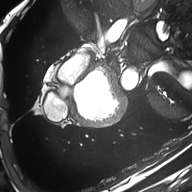}%
        \subcaption{}
        \label{sub-fig: clean-id5-t0}
      \end{minipage} 
    \end{tabular}
     \caption{
         Visualization of predicted images in Expt.2.
         The first to fifth rows correspond to ID~1 to ID~5, respectively.
         The first column corresponds to a noisy image at time $0$ except for the second row;~\eqref{sub-fig: noisy-id2-t1} visualizes the noisy image at time $1$. The second to fifth columns correspond to predicted image by TD, predicted image by D2S, predicted image by DN2N, and clean image at time $0$, respectively.
     }
     \label{fig: append acdc before after}
  \end{figure*}

The quantitative results are in Table~\ref{tab:acdc results}. From the table, we can see that our method performs overall the best, and the second best method is D2S. In contrast to Table~\ref{tab:toydata results}, our method outperforms both BM3D and N2F, indicating that DN2N is empirically efficient on a weak denaturation dataset, whose reference image $\bm{y}_0$ has a relatively complicated structure.

The visualization of the predicted images generated by our method is shown in Figure~\ref{fig: append acdc before after}. For example, by zoom-in, we can observe that the prediction of~\eqref{subfig: dn2n-pred-id3} is sufficiently denoised compared to~\eqref{subfig: noisy-id3-t0} and comparable to the clean counterpart~\eqref{subfig: clean-id3-t0}. Furthermore, by comparing~\eqref{subfig: d2s-pred-id3} and~\eqref{subfig: clean-id3-t0} in terms of the remaining noise, we can see that our method performs better for the sub-image than D2S.

\subsection{Expt.3: Performance Evaluation using Cryo-EM Image Dataset}
\label{subsec: expt3}
\subsubsection{Setting}
\label{subsubsec: setting expt3}
We employ a multiframe file containing 16 Cryo-EM images of multiple particles from the dataset called EMPIAR-10028~\cite{Iudin2023-ua}. Let us define the time series raw 16 Cryo-EM images by $\Tilde{\mathcal{D}}=(\bm{y}_0,...,\bm{y}_i,...,\bm{y}_{N})$, where $N=15$, and $i\in\{0\}\cup\mathbb{N}$ is the time index.
According to~\cite{Iudin2023-ua}, all the images are of size 4096$\times$4096, and are contaminated by noise.
Since the original image size is too large for our computational resources to train our model, we firstly select two different small areas (both sizes are 192$\times$192) containing single particle in $\bm{y}_0$, using the average Cryo-EM image $\Bar{\bm{x}} :=\frac{1}{N+1}\sum_{i=0}^{N} \bm{y}_i$.  
The selected two areas in $\bm{y}_0$ are visualized as two yellow squared regions in~\eqref{subfig: ogirinal sub image}; see also zoom-in images for the two areas in~\eqref{subfig: original-top-left} and~\eqref{subfig: original-bottom-right}. The corresponding sub images to~\eqref{subfig: ogirinal sub image} in $\Bar{\bm{x}}$ 
is visualized in~\eqref{subfig: average sub image}. For each Cryo-EM image in $\bm{y}_0,...,\bm{y}_{N}$, we clip two small square areas based on the coordinates of the two yellow squared areas in~\eqref{subfig: ogirinal sub image} to make training datasets of DN2N.

We denoise the two small areas of~\eqref{subfig: original-top-left} and~\eqref{subfig: original-bottom-right} by using our trained model, then conduct visual inspection to evaluate the predictions. For defining $\bm{y}_{\tau_i, j}$ in~\eqref{eq:empirical risk}, we follow the same protocol shown in Section~\ref{subsubsec: expt1 setting}, using a set of the 16 clipped images. 
In addition, we employ TD as the baseline method.

\begin{rmk}
\label{rmk: motivation of expt3}
    An averaged image over raw Cryo-EM images is commonly used as a higher-resolution reference to the raw images, when reconstructing the 3D protein structures by \texttt{RELION}~\cite{scheres2012relion} or \texttt{cryoSPARC}~\cite{punjani2017cryosparc}.
    The goal of Expt.3 is to check whether our method can be an efficient preprocessing technique before the averaging.
    If it is efficient, our method has a potential to enhance 
    quality of the averaged image, leading more accurate 3D reconstruction.
\end{rmk}


\paragraph{Supplementary Information for Figure~\ref{fig: cryowm before after}}
We provide coordinates' information with~\eqref{subfig: ogirinal sub image} in Figure~\ref{fig: cryowm before after}. First, \eqref{subfig: ogirinal sub image} is a sub-image in the original Cryo-EM image at time $0$, i.e., $\bm{y}_0$, whose size is 4096$\times$4096. 
Following the introduction of the coordinates in Section~\ref{subsubsec: expt1 setting}, 
the top-left, bottom-left, bottom-right, and top-right corners' coordinates of $\bm{y}_0$ are $(1,1)$, $(1,4096)$, $(4096, 4096)$, and $(4096,1)$, respectively. 
We define the coordinates of the bottom-left, bottom-right, top-right, and top-left corners of the sub-image~\eqref{subfig: ogirinal sub image} in $\bm{y}_0$ by 
$A,B,C,\;\textrm{and}\;D$, respectively. Moreover, we define the coordinates of the top-left corner point in the top-left (resp. bottom-right) yellow square of~\eqref{subfig: ogirinal sub image} in $\bm{y}_0$ by $E$ (resp. $F$). In this case, $A=(593, 3836), B=(1316, 3836), C=(1316, 2925),D=(593, 2925), E=(641, 2981),\;\textrm{and}\;F=(1085, 3605)$. 

At~\eqref{sub-fig: cryoem-original-top-left-t1} to~\eqref{sub-fig: cryoem-original-bottom-right-t15} in Figure~\ref{fig: training dataset for each expt}, we visualize how the 16 clipped Cryo-EM images are denatured. In the figure, the left to right image in the fourth (resp. fifth) row correspond to the clipped images of top-left (resp. bottom-right) yellow squared area in~\eqref{subfig: ogirinal sub image} with time index $i=1$, $4$, $8$, $12$, and $15$.

\subsubsection{Result and Analysis}

\begin{figure*}[!t]
    \centering
    \begin{minipage}{.25\linewidth}
    \centering
        \begin{subfigure}[b]{\linewidth}
            \includegraphics[width=\textwidth]{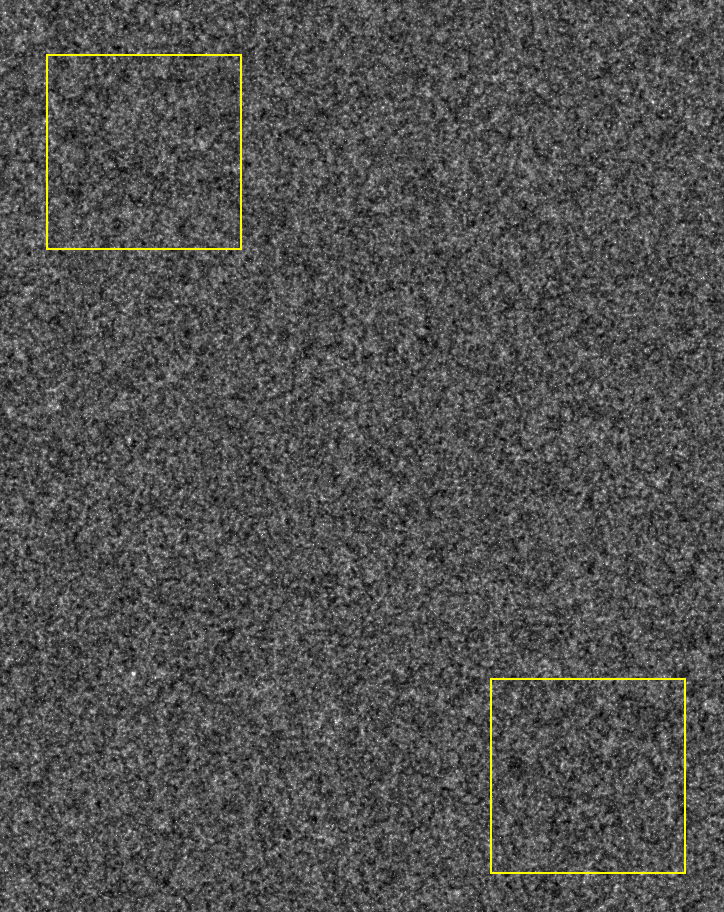}
            \caption{}
            \label{subfig: ogirinal sub image}
        \end{subfigure}
    \end{minipage}
    \begin{minipage}{.25\linewidth}
    \centering
        \begin{subfigure}[b]{\linewidth}
            \includegraphics[width=\textwidth]{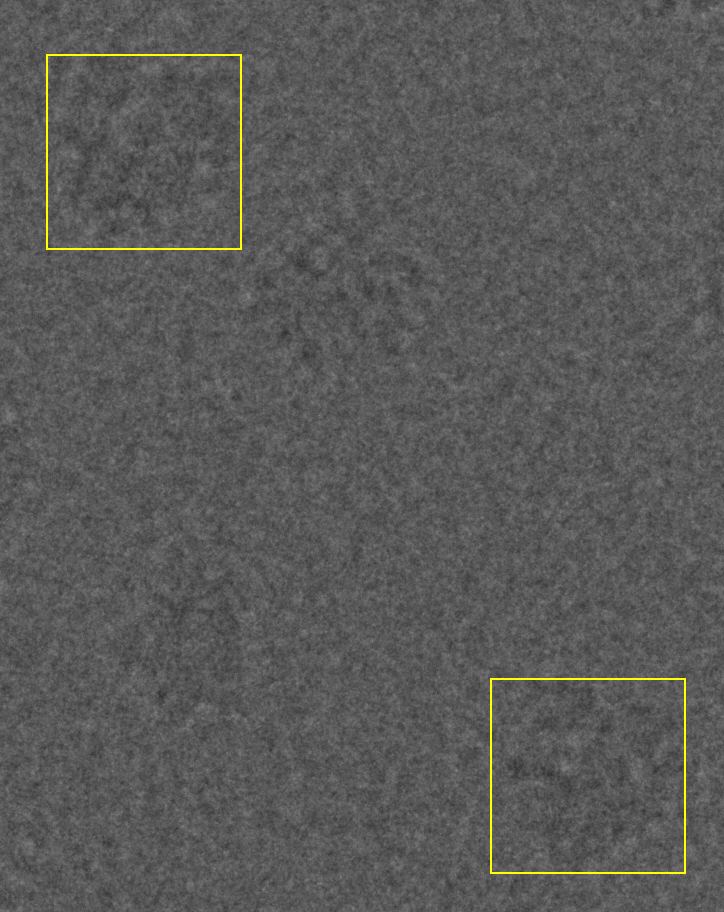}
            \caption{}
            \label{subfig: average sub image}
        \end{subfigure}
    \end{minipage}
    \begin{minipage}{.415\linewidth}
    \centering
        \begin{tabular}{ccc}
                \begin{subfigure}[b]{.325\linewidth}
                \centering
                    \includegraphics[width=0.8\hsize]{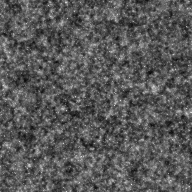}
                    \caption{}
                    \label{subfig: original-top-left}
                \end{subfigure} &
                \begin{subfigure}[b]{.325\linewidth}
                \centering
                    \includegraphics[width=0.8\hsize]{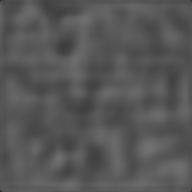}
                    \caption{}
                    \label{subfig: prediction id3}
                \end{subfigure} &
                \begin{subfigure}[b]{.325\linewidth}
                \centering
                    \includegraphics[width=0.8\hsize]{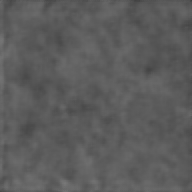}
                    \caption{}
                    \label{subfig: td-pred-top-left}
                \end{subfigure} 
                \\
                \begin{subfigure}[b]{.325\linewidth}
                \centering
                    \includegraphics[width=0.8\hsize]{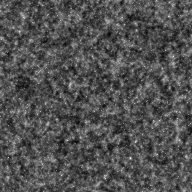}
                    \caption{}
                    \label{subfig: original-bottom-right}
                \end{subfigure} &
                \begin{subfigure}[b]{.325\linewidth}
                \centering
                    \includegraphics[width=0.8\hsize]{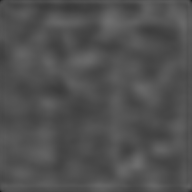}
                    \caption{}
                    \label{subfig: prediction id2}
                \end{subfigure} &
                \begin{subfigure}[b]{.325\linewidth}
                \centering
                    \includegraphics[width=0.8\hsize]{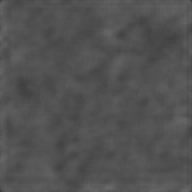}
                    \caption{}
                    \label{subfig: td-pred-bottom-right}
                \end{subfigure}
        \end{tabular}
    \end{minipage}
    \caption{
            Visualized results of Expt.3: application of our method to two small yellow squared areas in~\eqref{subfig: ogirinal sub image}. 
            \eqref{subfig: ogirinal sub image}: 
            Sub-image in a raw Cryo-EM image at time $0$. 
            \eqref{subfig: average sub image}: 
            The corresponding sub-image to~\eqref{subfig: ogirinal sub image} in the averaged Cryo-EM image. 
            \eqref{subfig: original-top-left} and~\eqref{subfig: original-bottom-right}:
            The zoom-in images in top-left and bottom-right yellow squares of~\eqref{subfig: ogirinal sub image}, respectively; both are clipped by \texttt{Fiji}~\cite{Schindelin2012-ua}.
            \eqref{subfig: prediction id3} and~\eqref{subfig: prediction id2}:
            Predicted images by DN2N for~\eqref{subfig: original-top-left} and~\eqref{subfig: original-bottom-right}, respectively.
            \eqref{subfig: td-pred-top-left} and~\eqref{subfig: td-pred-bottom-right}: 
            Predicted images by TD for~\eqref{subfig: original-top-left} and~\eqref{subfig: original-bottom-right}, respectively. 
    }
    \label{fig: cryowm before after} 
\end{figure*}

The results are presented in Figure~\ref{fig: cryowm before after}. Let us treat the average image $\Bar{\bm{x}}$ as a higher-resolution reference for $\bm{y}_0$; see Remark~\ref{rmk: motivation of expt3}. 
First, we can faintly see that each area surrounded by yellow squares in the average sub-image~\eqref{subfig: average sub image} contains a circular particle, whereas we can see almost no particle in the same areas of~\eqref{subfig: ogirinal sub image}.
Next, we can faintly see a circular-shaped particle in the prediction~\eqref{subfig: prediction id3} (resp.~\eqref{subfig: prediction id2}) and the particle is similar to one in the top-left (resp. bottom-right) yellow squared region of~\eqref{subfig: average sub image}.  
Furthermore, the particles in~\eqref{subfig: prediction id3} and~\eqref{subfig: prediction id2} are relatively more visible than the ones in~\eqref{subfig: td-pred-top-left} and~\eqref{subfig: td-pred-bottom-right} obtained through TD.

\subsection{Expt.4: Performance Evaluation using Fluorescence Microscopy Image Dataset}
\label{subsec: expt4}
Fluorescence Microscopy (FM) is an important device in biology, since it has contributed to the understanding of biological structures and functions~\cite{Sanderson01102014}. In this experiment, we employ raw FM images of zebrafish embryos, which have been studied well in drug discovery~\cite{macRae2015zebrafish}. The images were collected by the authors of~\cite{8953965}, and the time series raw 50 FM images used in this experiment can be found in folder ``raw/2'' inside ``Confocal\_FISH.tar'' available at the following URL~\cite{haward2020n2FMdataset}.

\subsubsection{Setting}
\label{subsubsec: setting expt4}
As briefly introduced above, we employ the 50 FM images of zebrafish embryos. Let $\Tilde{\mathcal{D}}=(\bm{y}_0,...,\bm{y}_i,...,\bm{y}_{49})$ denote the raw dataset, where each size is 512$\times$512; see visualization of $\bm{y}_0$ in~\eqref{subfig: mfm-full-noisy-t0}. For each raw image, using \texttt{Fiji}~\cite{Schindelin2012-ua}, we clip the 192$\times$192 sub-image corresponding to the yellow squared area in~\eqref{subfig: mfm-full-noisy-t0}, and define a sequence of the clipped 50 sub-images as training dataset for DN2N. Let $\tilde{\mathcal{D}}^{\mathrm{clip}}$ denote the training dataset.

\paragraph{Supplementary Information for Figure~\ref{fig: mfm before after}}
Let us define the top-left, bottom-left, bottom-right, and top-right corners’ coordinates of $\bm{y}_0$ are $(1, 1)$, $(1, 512)$, $(512, 512)$, and $(512, 1)$, respectively. In the definition, the top-left, bottom-left, bottom-right, and top-right corners’ coordinates of the yellow square in \eqref{subfig: mfm-full-noisy-t0} are $(184,88)$, $(184,279)$, $(375,279)$, $(375,88)$, respectively.

\begin{figure*}[!t]
    \begin{tabular}{ccc}
            \begin{subfigure}[b]{.3\linewidth}
            \centering
                \includegraphics[width=0.65\hsize]{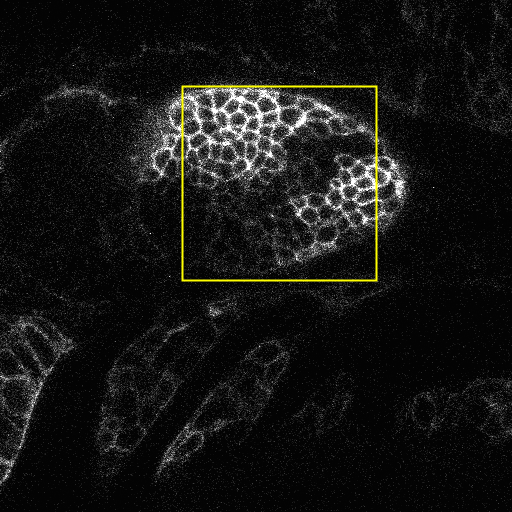}
                \caption{}
                \label{subfig: mfm-full-noisy-t0}
            \end{subfigure} &
            \begin{subfigure}[b]{.3\linewidth}
            \centering
                \includegraphics[width=0.65\hsize]{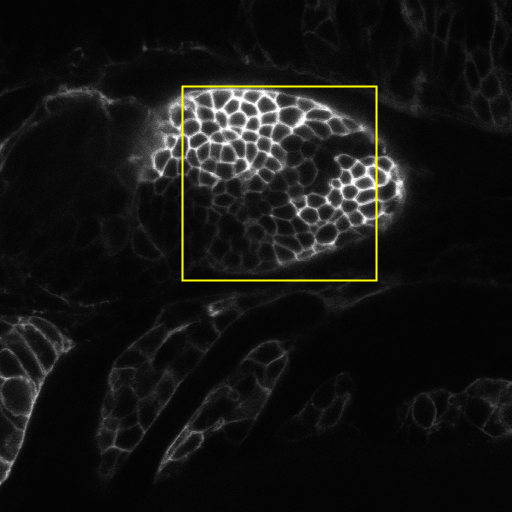}
                \caption{}
                \label{subfig: mfm-full-average}
            \end{subfigure} &
            \begin{subfigure}[b]{.3\linewidth}
            \centering
                \includegraphics[width=0.65\hsize]{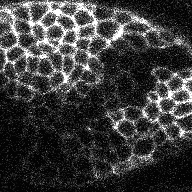}
                \caption{}
                \label{subfig: mfm-clipped-noisy-t0}
            \end{subfigure} \\
            \begin{subfigure}[b]{.3\linewidth}
            \centering
                \includegraphics[width=0.65\hsize]{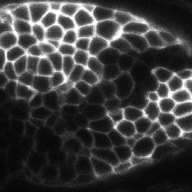}
                \caption{}
                \label{subfig: mfm-clipped-average}
            \end{subfigure} &
            \begin{subfigure}[b]{.3\linewidth}
            \centering
                \includegraphics[width=0.65\hsize]{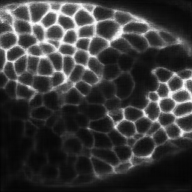}
                \caption{}
                \label{subfig: mfm-dn2n-pred}
            \end{subfigure} &
            \begin{subfigure}[b]{.3\linewidth}
            \centering
                \includegraphics[width=0.65\hsize]{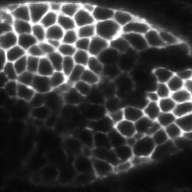}
                \caption{}
                \label{subfig: mfm-td-pred}
            \end{subfigure} 
    \end{tabular}
    \caption{
            Visualized results of Expt.4: application of DN2N to sub-image inside of yellow squared area in~\eqref{subfig: mfm-full-noisy-t0}. 
            \eqref{subfig: mfm-full-noisy-t0}: Raw Fluorescence microscopy image at $t=0$, and the yellow squared area is clipped for Expt.4,
            \eqref{subfig: mfm-full-average}: The corresponding averaged image to \eqref{subfig: mfm-full-noisy-t0}, 
            \eqref{subfig: mfm-clipped-noisy-t0} and~\eqref{subfig: mfm-clipped-average}: The clipped images by \texttt{Fiji} from yellow squared areas of~\eqref{subfig: mfm-full-noisy-t0} and~\eqref{subfig: mfm-full-average}, respectively,
            \eqref{subfig: mfm-dn2n-pred} and \eqref{subfig: mfm-td-pred}: 
            Predicted images output from the models trained by DN2N and TD, respectively, where the input image is~\eqref{subfig: mfm-clipped-noisy-t0}.
    }
    \label{fig: mfm before after} 
\end{figure*}

At~\eqref{sub-fig: mfm-noisy-t0} to~\eqref{sub-fig: mfm-noisy-t49} in Figure~\ref{fig: training dataset for each expt}, we visualize how the dataset $\tilde{\mathcal{D}}^{\mathrm{clip}}$ denatures. The left to right image in the sixth row of the figure correspond to the clipped images with time index $i=1$, $13$, $25$, $38$, and $49$, respectively. 

We conduct visual inspection for predicted image by DN2N. 
For the inspection, inspired by~\cite{8953965}, we use the averaged FM image $\Bar{\bm{x}} :=\frac{1}{50}\sum_{i=0}^{49} \bm{y}_i$ 
as a higher-resolution reference to $\bm{y}_0$; see~\eqref{subfig: mfm-full-average} for visualization of $\Bar{\bm{x}}$. We follow the same protocol of Section~\ref{subsubsec: expt1 setting} to define $\bm{y}_{\tau_i, j}$ in~\eqref{eq:empirical risk} based on $\tilde{\mathcal{D}}^{\mathrm{clip}}$. In addition, for comparison, we conduct the same experiment with TD described in Section~\ref{subsubsec: expt1 setting}.

\subsubsection{Result and Analysis}
Firstly, by comparing the predicted image~\eqref{subfig: mfm-dn2n-pred} of DN2N with the raw counterpart~\eqref{subfig: mfm-clipped-noisy-t0}, it is clear that the noise is significantly reduced overall. As for comparison with the predicted image~\eqref{subfig: mfm-td-pred} by TD, the boundary region with low fluorescence intensity from the top-right to the bottom-left in~\eqref{subfig: mfm-dn2n-pred}, which is informative for understanding the process of membrane formation, is less noisy than~\eqref{subfig: mfm-td-pred}, while the top-left region with high fluorescence intensity in~\eqref{subfig: mfm-td-pred} is slightly clearer than~\eqref{subfig: mfm-dn2n-pred}. Furthermore, by comparing~\eqref{subfig: mfm-dn2n-pred} and the average image~\eqref{subfig: mfm-clipped-average}, the bottom-right region of~\eqref{subfig: mfm-dn2n-pred} with high fluorescence intensity is less noisy than~\eqref{subfig: mfm-clipped-average}, and the boundary region from the center to the bottom-left in~\eqref{subfig: mfm-dn2n-pred} is slightly clearer than~\eqref{subfig: mfm-clipped-average}. Those evidences indicate that DN2N has the potential to positively impact biological analysis, such as how embryos form tissues. 

\begin{rmk}
\label{rmk: pnsr-ssim-expt4}
    Although the higher PSNR/SSIM value between~\eqref{subfig: mfm-dn2n-pred} and~\eqref{subfig: mfm-clipped-average} does not necessarily guarantee that the predicted image by DN2N is useful for biological analysis, we report the results for completeness. Regarding the results for DN2N, the PSNR/SSIM between~\eqref{subfig: mfm-dn2n-pred} and~\eqref{subfig: mfm-clipped-average} are 31.36/0.954. On the other hand, regarding the ones for TD, the PSNR/SSIM between~\eqref{subfig: mfm-td-pred} and~\eqref{subfig: mfm-clipped-average} are 27.98/0.890.
\end{rmk}

\subsection{Implementation and Hyperparameter Tuning}
\label{subsec:implementation}
We implement the algorithms with \texttt{PyTorch}~\cite{paszke2019pytorch}.
For the architecture, we employ the model used in~\cite{lehtinen2018noise}, which is based on the architecture called U-Net~\cite{ronneberger2015unet}.
Here, since our algorithm requires the time variable, we create additional one-channel tensor directly using \texttt{PyTorch}~\cite{paszke2019pytorch}, whose shape is $W\times H$ and each entry is equal to the value of the time variable, and we concatenate it with each $C$-channel training image to produce a $(C+1)$-channel image of shape $W\times H$.
Note that following~\cite{lehtinen2018noise}, we divide the loss function by the number of pixels in each image input to the networks.

\begin{table*}[!t]
\caption{
Results of prediction performance by DN2N with the different hyperparameter values $(\Tilde{\sigma},\mu)$ on the ACDC dataset defined in Expt.2, while fixing $(L,K, n_{\mathrm{epochs}})$ to $(2,100,1000)$.
Note that the result for the selection $(\tilde{\sigma}, \mu) = (75, 10)$ is shown in the last row of Table~\ref{tab:acdc results}.
}
\label{tab:acdc hyperparameter tuning results}
    \centering
    \scalebox{0.8}{
    \begin{tabular}{cccccccccccc}
        \toprule
            \multirow{2}{*}{$\tilde{\sigma}$} & \multirow{2}{*}{$\mu$}  & \multicolumn{2}{c}{ID 1}  & \multicolumn{2}{c}{ID 2}  & \multicolumn{2}{c}{ID 3}  & \multicolumn{2}{c}{ID 4}  & \multicolumn{2}{c}{ID 5} \\ \cline{3-12}
             &     & PSNR$\uparrow$  & SSIM$\uparrow$ & PSNR & SSIM & PSNR & SSIM & PSNR & SSIM & PSNR & SSIM \\ \hline
            \multirow{5}{*}{0}
            & 0 & 22.96 & 0.762 & 24.26 & 0.773 & 24.10 & 0.754 & 21.61 & 0.786 & 22.81 & 0.725 \\
            & 0.1 & 22.94 & 0.779 & 24.03 & 0.760 & 24.46 & 0.773 & 21.69 & 0.789 & 22.71 & 0.711 \\ 
            & 1 & 22.98 & 0.781 & 23.98 & 0.755 & 24.59 & 0.781 & 21.44 & 0.778 & 22.89 & 0.733 \\ 
            & 10 & 23.00 & 0.776 & 24.45 & 0.774 & 25.39 & 0.817 & 21.69 & 0.791 & 22.86 & 0.732  \\ 
            & 100 & 23.10 & 0.813 & 24.97 & 0.792 & 25.55 & 0.837 & 22.00 & 0.800 & 23.01 & 0.758 \\ \hline 
            \multirow{5}{*}{50}
            & 0 & 22.81 & 0.787 & 24.15 & 0.765 & 24.96 & 0.808 & 20.81 & 0.765 & 22.75 & 0.753 \\
            & 0.1 & 23.15 & 0.789 & 24.62 & 0.767 & 24.92 & 0.817 & 20.78 & 0.768 & 22.95 & 0.763 \\ 
            & 1 & 23.01 & 0.795 & 24.28 & 0.763 & 24.82 & 0.825 & 21.26 & 0.778 & 22.79 & 0.744 \\ 
            & 10 & 23.06 & 0.800 & 24.54 & 0.782 & 25.42 & 0.830 & 21.36 & 0.785 & 22.90 & 0.779  \\ 
            & 100 & 22.87 & 0.802 & 24.89 & 0.801 & 25.54 & 0.840 & 21.31 & 0.782 & 22.90 & 0.773 \\ \hline
            \multirow{4}{*}{75}
            & 0 & 22.99 & 0.792 & 23.84 & 0.768 & 25.48 & 0.834 & 20.91 & 0.769 & 22.66 & 0.753 \\
            & 0.1 & 23.10 & 0.792 & 24.34 & 0.780 & 25.51 & 0.838 & 21.02 & 0.775 & 22.85 & 0.766 \\ 
            & 1 & 23.16 & 0.796 & 24.53 & 0.771 & 25.30 & 0.827 & 20.99 & 0.768 & 22.99 & 0.752  \\ 
            & 100 & 22.73 & 0.802 & 24.21 & 0.797 & 25.39 & 0.842 & 21.25 & 0.779 & 22.79 & 0.770  \\ \hline
            \multirow{5}{*}{100}
            & 0 & 23.23 & 0.799 & 24.20 & 0.767 & 25.09 & 0.828 & 21.07 & 0.776 & 22.97 & 0.753 \\
            & 0.1 & 23.23 & 0.802 & 24.16 & 0.770 & 25.39 & 0.834 & 20.96 & 0.775 & 22.94 & 0.751  \\ 
            & 1 & 23.06 & 0.793 & 24.52 & 0.777 & 25.41 & 0.842 & 20.55 & 0.765 & 23.07 & 0.768  \\
            & 10 & 22.80 & 0.801 & 24.07 & 0.785 & 24.97 & 0.839 & 21.23 & 0.780 & 22.64 & 0.771 \\
            & 100 & 22.52 & 0.800 & 24.84 & 0.798 & 25.26 & 0.841 & 20.74 & 0.771 & 22.68 & 0.769 \\
            \bottomrule
        \end{tabular}
    }
\end{table*}

In our method, there are the following main hyperparameters: $\Tilde{\sigma}$, $\mu$, $L$, and $K$. 
The hyperparameter $\Tilde{\sigma}$ represents the standard deviation of the multivariate zero-mean normal distribution, from which the auxiliary noise $\bm{\varepsilon}$ in~\eqref{eq:empirical risk} is sampled. The hyperparameter $\mu$ is used to define $\widehat{\mathcal{L}}_{T}$ in Section~\ref{sec:proposed methodology}. The quantities $L$ and $K$ respectively denote the number of transformations in~\eqref{eq:empirical risk} and the number of reconstructed images used for the prediction $\widehat{x}_{0}$ in~\eqref{eq: prediction formula}. 
We set $(\Tilde{\sigma}, \mu, L, K)$ to $(100, 100, 2, 100)$ in Expt.1, to $(75, 10, 2, 100)$ in Expt.2, to $(75, 10, 2, 100)$ in Expt.3, and to $(75, 100, 2, 100)$ in Expt.4.
The predicting performance of DN2N is robust against change of $(\tilde{\sigma}, \mu)$; see Table~\ref{tab:acdc hyperparameter tuning results}.
For training our model in all experiments, all images are normalized, i.e., each pixel value in an image is divided by 255. We use the optimizer called Adam~\cite{kingma2015adam}, where we set the learning rate to 0.0001.
The mini-batch size is fixed to 4 through all the experiments.
Moreover, we set the number of epochs for training model to 1000 (resp. 900) in Expt.1, Expt.2, and Expt.4 (resp. Expt.3).

We report results of ablation study on DN2N using the same ACDC dataset of Expt.2 in Table~\ref{tab:acdc hyperparameter tuning results}. In the study, we consider two cases: $\mu=0$ and $\tilde{\sigma}=0$. In the first case, we do not use the regularizer $\widehat{\mathcal{L}}_{A}(f)$ of Section~\ref{sec:proposed methodology}, while, in the second case, we do not employ the auxiliary noise $\bm{\varepsilon}$ for both training and prediction. By comparing $\mu=0$ and $\mu \neq 0$ in the table, we can see clear positive effect of $\widehat{\mathcal{L}}_{A}(f)$. Additionally, the SSIM values with $\tilde{\sigma}\neq0$ are overall higher than those with $\tilde{\sigma}=0$, while 
we cannot see the large difference of the PSNR values between the cases with
$\tilde{\sigma}\neq0$  and
$\tilde{\sigma}=0$.

\section{Conclusion}
\label{sec:conclusion}
We theoretically analyze the prediction performance of a self-supervised denoising method on noisy images with denaturation. Based on the analytical result in Theorem~\ref{thm:non-asymptotic theory}, we design DN2N. In our numerical experiments, we observe that the performance of DN2N on the synthetic dataset is consistent with the analytical result. Further, DN2N empirically performs well for MRI, Cryo-EM, and FM image datasets.

We conclude this paper by presenting several interesting future works.
\begin{itemize}
\item We firstly point out that relaxing Assumption~\ref{assumption:assumptions} to more general assumptions is crucial for applications.
\item In addition, constructing more efficient ways to reduce the influence from the hardness of denaturation levels is also important to make full use of denatured images.
\item We comment on the computational efficiency of the algorithm. The algorithm is based on an MSE-based loss function, following~\cite{lehtinen2018noise}.
Hence, it seems that the computational cost for optimizing the loss function is similar to the cost for the least squares problem proposed by~\cite{lehtinen2018noise}.
On the other hand, compared to other methods proposed in~\cite{quan2020self,lequyer2022noise}, our algorithm needs to use more images.
Thus, our algorithm suffers from an issue on the high memory cost to run the algorithm.
In the case that the number of images is small, this issue can be avoided.
However, another problem on the trade-off between the computational aspect and generalization can occur, as shown in Theorem~\ref{thm:non-asymptotic theory}.
Therefore, in practice we need to optimize the number of images, depending on the properties of the given training dataset.
\item Another interesting future work is to investigate whether self-supervised denoising can avoid curse of dimensionality, which is a common issue for large-scaled learning problems in machine learning. We leave this problem as a future work since it is not what we mainly focus on in this paper.
Nevertheless, to resolve the problem, it might be possible to utilize the theory of nonparametric estimation using deep neural networks (see e.g.,~\cite{suzuki2018adaptivity,chen2022nonparametric}).
\end{itemize}

\section*{Data Availability}
Some of our numerical results in Section~\ref{sec:experiments} can be reproduced at URL\footnote{\url{https://doi.org/10.5281/zenodo.14348971}}.

\section*{Acknowledgments}
This work was supported by the FOCUS Establishing Supercomputing Center of Excellence project; MEXT as ``Program for Promoting Researches on the Supercomputer Fugaku'' (Simulation- and AI-driven next-generation medicine and drug discovery based on ``Fugaku'', JPMXP1020230120). This work used computational resources of the supercomputer Fugaku provided by RIKEN Center for Computational Science through the HPCI System Research Projects (Project IDs: hp220078, hp230102, hp230216, hp240109, hp240211,ra000018); the supercomputer system at the information initiative center, Hokkaido University, Sapporo, Japan through the HPCI System Research Projects (Project IDs: hp220078, hp230102). Computational resources of AI Bridging Cloud Infrastructure (ABCI) provided by the National Institute of Advanced Industrial Science and Technology (AIST) was also used. We would like to thank Yasushi Okuno, Akira Nakagawa, and Takashi Kato for their feedback on the draft.

\appendix

\section{Proofs}
\label{appsubsec:proofs}
In this section, we provide all the proofs omitted in the main paper.

\subsection{Proof of Proposition~\ref{thm:guarantee for population minimizer}}
\label{appsubsec:proof of proposition 1}

\begin{proof}
In the first step, we derive the optimal solution of the least-squares problem.
We follow a similar way to Section~2.4 in~\cite{hastie2009elements} for solving the least-squares problem.
Note that
\begin{align*}
\mathbb{E}_{\bm{y},\tau}[\|f(\bm{y}_{0},\tau)-\bm{y}_{\tau}\|_{2}^{2}|\bm{x}_{0}]
=
\mathbb{E}_{\bm{y}_{0},\tau}[\mathbb{E}_{\bm{y}_{\tau}}[\|f(\bm{y}_{0},\tau)-\bm{y}_{\tau}\|_{2}^{2}|\tau,\bm{y}_{0}]|\bm{x}_{0}].
\end{align*}
For a fixed $\tau$, note that $\bm{x}_{0}$ has been also fixed.
Then, it holds that
\begin{align*}
&\mathbb{E}_{\bm{y}_{\tau}}[\|f(\bm{y}_{0},\tau)-\bm{y}_{\tau}\|_{2}^{2}|\tau,\bm{y}_{0}]\\
&=\sum_{i=1}^{d}\mathbb{E}_{\bm{y}_{\tau}}[f_{i}^{2}(\bm{y}_{0},\tau)-2\bm{y}_{\tau,i}f(\bm{y}_{0},\tau)_{i}+\bm{y}_{\tau,i}^{2}|\tau,\bm{y}_{0}]\\
&=\sum_{i=1}^{d}\{(f_{i}(\bm{y}_{0},\tau)-\mathbb{E}_{\bm{y}_{\tau}}[\bm{y}_{\tau,i}|\tau,\bm{y}_{0}])^{2}-(\mathbb{E}_{\bm{y}_{\tau}}[\bm{y}_{\tau,i}|\tau,\bm{y}_{0}])^{2}+\mathbb{E}_{\bm{y}_{\tau}}[\bm{y}_{\tau,i}^{2}|\tau,\bm{y}_{0}]\},
\end{align*}
where $f_{i}(\bm{y}_{0},\tau)$ and $\bm{y}_{\tau,i}$ denote the $i$-th index of the vectors, respectively.
This indicates that $f^{*}(\bm{y}_{0},\tau)=\mathbb{E}_{\bm{y}_{\tau}}[\bm{y}_{\tau}|\tau,\bm{y}_{0}]$ almost surely.

Now, we consider to fix $\bm{y}_{0}\in\mathbb{R}^{d}$ for convenience.
For $t>0$, letting $t\downarrow 0$,
\begin{align*}
    \lim_{t\downarrow 0}f^{*}(\bm{y}_{0},t)
    =\lim_{t\downarrow 0}\phi_{t}(\bm{x}_{0})
    =\bm{x}_{0},
\end{align*}
where we use in the first equality that $\mathbb{E}_{\bm{y}_{t}}[\bm{y}_{t}|\bm{y}_{0}]=\phi_{t}(\bm{x}_{0})$ from the conditions \textbf{\textup{(A)}} and \textbf{\textup{(C)}} in Assumption~\ref{assumption:assumptions}, and in the second one, the condition \textbf{\textup{(B)}} in Assumption~\ref{assumption:assumptions}.
Here, from the assumption, it holds that $\lim_{t\downarrow 0}f^{*}(\bm{y}_{0},t)=f^{*}(\bm{y}_{0},0)$.
Therefore, $f^{*}(\bm{y}_{0},0)=\bm{x}_{0}$ holds almost every $\bm{y}_{0}$ with its probability measure.
This implies that $\mathbb{E}_{\bm{y}_{0}}[f^{*}(\bm{y}_{0},0)|\bm{x}_{0}]=\bm{x}_{0}$, as claimed.
\hfill
\end{proof}

\subsection{Proof of Theorem~\ref{thm:non-asymptotic theory}}
\label{appsubsec:proof of theorem 1}
We note that the following fact holds.
\begin{lem}
\label{lem:uniform boundedness}
Suppose that the conditions in the statement of Theorem~1 hold.
Let $V=\{\bm{v}\in\mathbb{R}^{d}\;|\;\|\bm{v}\|_{2}\leq B_{2}\}$.
Then for any $f\in\mathcal{F}$, it holds that 
$$
\sup_{\bm{u}\in U,\bm{v}\in V,t\in (0,1]}\|f(\bm{u},t)-\bm{v}\|_{2}^{2}\leq L^{2}(B_{1}^{2}+1)+B_{2}^{2}.
$$
\end{lem}
\begin{proof}
    First note that
    \begin{align*}
        \sup_{\bm{u}\in U,\bm{v}\in V,t\in (0,1]}\|f(\bm{u},t)-\bm{v}\|_{2}^{2}\leq 
        \left(\sup_{\bm{u}\in U,t\in (0,1]}\|f(\bm{u},t)\|_{2}+\sup_{\bm{v}\in V}\|\bm{v}\|_{2}\right)^{2}.
    \end{align*}
    The second term in the upper bound is further bounded by $B_{2}$ from the condition on $V$.
    For the first part, notice that $f(\bm{0},0)=0$ holds and $f$ is $L$-Lipschitz, we have
    \begin{align*}
        \sup_{\bm{u}\in U,t\in (0,1]}\|f(\bm{u},t)\|_{2}\leq L\sqrt{\|\bm{u}\|_{2}^{2}+|t|^{2}}\leq L\sqrt{B_{1}^{2}+1}.
    \end{align*}
    Hence, we obtain
    \begin{align*}
        \sup_{\bm{u}\in U,\bm{v}\in V,t\in (0,1]}\|f(\bm{u},t)-\bm{v}\|_{2}^{2}\leq \left(L\sqrt{B_{1}^{2}+1}+B_{2}\right)^{2},
    \end{align*}
    which shows the claim.
    \hfill
\end{proof}

Following \cite[Definition~3.2]{mohri2018foundations}, we define a variant of \emph{Rademacher complexity}~\cite{bartlett2002rademacher} as
\begin{align*}
    \mathcal{R}_{j}(\mathcal{F})=\mathbb{E}_{\tau,\sigma}\left[\sup_{f\in\mathcal{F}}\frac{1}{N}\sum_{i=1}^{N}\sigma_{i}\|f(\bm{y}_{0,j},\tau_{i})-\bm{y}_{\tau_{i},j}\|_{2}^{2}\right],
\end{align*}
where $\{\sigma_{i}\}$ are i.i.d. random variables such that $\sigma_{i}=+1$ holds with probability 1/2 and $\sigma_{i}=-1$ holds with probability 1/2, and the expectation is taken for all random variables except $\bm{y}$.
To prove Theorem~1, we need the following inequality, whose proof is based on \cite[Theorem~3.3]{mohri2018foundations}.
\begin{lem}
\label{lem:uniform convergence for u-statistics}
In the setting described in the statement of Theorem~1, with probability at least $1-2\delta$ where $\delta\geq 0$, it holds that
\begin{align*}
    &\frac{1}{MN}\sum_{j=1}^{M}\sum_{i=1}^{N}\|f(\bm{y}_{0,j},\tau_{i})-\bm{y}_{\tau_{i},j}\|_{2}^{2}\\
    &\leq \mathbb{E}\left[\|f(\bm{y}_{0},\tau_{i})-\bm{y}_{\tau_{i}}\|_{2}^{2}\right]+2\frac{1}{M}\sum_{j=1}^{M}\mathcal{R}_{j}(\mathcal{F})\\
    &\quad+\sqrt{\frac{2B^{2}\log(M/\delta)}{N}}+\sqrt{\frac{2B^{2}\log(1/\delta)}{M}}.
\end{align*}
\end{lem}
The proof of Lemma~\ref{lem:uniform convergence for u-statistics} can be found in~\ref{appsubsec:proof of lemma uniform convergence for u-statistics}.

In the proof of Theorem~1, we upper bound the following quantity.
\begin{lem}
\label{lem:rademacher complexity evaluation}
Suppose that the setting in Theorem~1 holds.
Then, there exists a constant $C>0$ such that we have
\begin{align*}
    \mathbb{E}_{\tau,\sigma}\left[\sup_{f\in\mathcal{F}}\frac{1}{N}\sum_{i=1}^{N}\sigma_{i}\|f(\bm{y}_{0,j},\tau_{i})\|_{2}\right]\leq C\frac{1}{\sqrt{N}}.
\end{align*}
\end{lem}
The proof of Lemma~\ref{lem:rademacher complexity evaluation} is deferred to~\ref{appsubsec:proof of lemma rademacher complexity evaluation}.
We also need the following lemma.
\begin{lem}
\label{lem:rademacher complexity evaluation version 2}
In the setting of Theorem~1, there is a positive constant $C$ such that we have
\begin{align*}
    \mathbb{E}_{\tau,\sigma}\left[\sup_{f\in\mathcal{F}}\frac{1}{N}\sum_{i=1}^{N}\sigma_{i}\langle f(\bm{y}_{0,j},\tau_{i}),\bm{y}_{\tau_{i},j}\rangle_{2}\right]\leq C\frac{1}{\sqrt{N}}.
\end{align*}
\end{lem}
The proof of Lemma~\ref{lem:rademacher complexity evaluation version 2} can be found in~\ref{appsubsec:proof of lemma rademacher complexity evaluation}.

\begin{proof}[Proof of Theorem~1]
We first remark that throughout this proof, the clean image $\bm{x}_{0}\in\mathbb{R}^{d}$ is fixed.
Using the triangle inequality, it holds that
\begin{align}
    &\left\|\widehat{\bm{x}}_{0}-\bm{x}_{0}\right\|_{2}^{2}\nonumber\\
    &=\left\|\widehat{\bm{x}}_{0}-\frac{1}{MN}\sum_{j=1}^{M}\sum_{i=1}^{N}\widehat{f}(\bm{y}_{0,j},\tau_{i})+\frac{1}{MN}\sum_{j=1}^{M}\sum_{i=1}^{N}\widehat{f}(\bm{y}_{0,j},\tau_{i})-\bm{x}_{0}\right\|_{2}^{2}\nonumber\\
    \label{eq:proof theorem 1 eq 1}
    &\leq 2\left\|\frac{1}{M}\sum_{j=1}^{M}\widehat{f}(\bm{y}_{0,j},0)-\frac{1}{MN}\sum_{j=1}^{M}\sum_{i=1}^{N}\widehat{f}(\bm{y}_{0,j},\tau_{i})\right\|_{2}^{2}\\
    &\quad+2\left\|\frac{1}{MN}\sum_{j=1}^{M}\sum_{i=1}^{N}\widehat{f}(\bm{y}_{0,j},\tau_{i})-\bm{x}_{0}\right\|_{2}^{2}\nonumber.
\end{align}
The first term in~\eqref{eq:proof theorem 1 eq 1} is upper bounded as
\begin{equation}
\label{eq:proof theorem 1 eq 2}
    \begin{split}
        &\left\|\frac{1}{M}\sum_{j=1}^{M}\widehat{f}(\bm{y}_{0,j},0)-\frac{1}{MN}\sum_{j=1}^{M}\sum_{i=1}^{N}\widehat{f}(\bm{y}_{0,j},\tau_{i})\right\|_{2}^{2}\\
        &\leq \frac{1}{MN}\sum_{j=1}^{M}\sum_{i=1}^{N}\left\|\widehat{f}(\bm{y}_{0,j},0)-\widehat{f}(\bm{y}_{0,j},\tau_{i})\right\|_{2}^{2}\\
        &\leq L^{2}\cdot\frac{1}{N}\sum_{i=1}^{N}\tau_{i}^{2},
    \end{split}
\end{equation}
where in the first inequality we use the triangle inequality, and in the second line we utilize the Lipschitz continuity of $\widehat{f}$ as assumed in the statement.
Here $\tau_{i}(\omega)\in (0,1]$ for every $\omega\in\Omega$ from the definition. By Hoeffding's inequality (see e.g., Theorem~D.2 in~\cite{mohri2018foundations}), for the mean $\frac{1}{N}\sum_{i=1}^{N}\tau_{i}^{2}$ we have with probability at least $1-\delta$,
\begin{align}
\label{eq:proof theorem 1 eq 2.5}
\frac{1}{N}\sum_{i=1}^{N}\tau_{i}^{2}\leq \mathbb{E}[\tau^{2}]+\sqrt{\frac{\log(1/\delta)}{2N}}.
\end{align}
By \eqref{eq:proof theorem 1 eq 2} and \eqref{eq:proof theorem 1 eq 2.5}, with probability at least $1-\delta$ we have
\begin{align}
\label{eq:proof theorem 1 eq 3}
    \left\|\frac{1}{M}\sum_{j=1}^{M}\widehat{f}(\bm{y}_{0,j},0)-\frac{1}{MN}\sum_{j=1}^{M}\sum_{i=1}^{N}\widehat{f}(\bm{y}_{0,j},\tau_{i})\right\|_{2}^{2}\leq L^{2}\cdot \left(\mathbb{E}[\tau^{2}]+\sqrt{\frac{\log(1/\delta)}{2N}}\right).
\end{align}

The second term in~\eqref{eq:proof theorem 1 eq 1} is evaluated as follows:
\begin{align}
\label{eq:proof theorem 1 eq 4}
    &\left\|\frac{1}{MN}\sum_{j=1}^{M}\sum_{i=1}^{N}\widehat{f}(\bm{y}_{0,j},\tau_{i})-\bm{x}_{0}\right\|_{2}^{2}\nonumber\\
    &\leq 
    2\left\|\frac{1}{MN}\sum_{j=1}^{M}\sum_{i=1}^{N}\widehat{f}(\bm{y}_{0,j},\tau_{i})\!-\!\frac{1}{MN}\sum_{j=1}^{M}\sum_{i=1}^{N}\bm{y}_{\tau_{i},j}\right\|_{2}^{2}\nonumber\\
    &\quad+2\left\|\frac{1}{MN}\sum_{j=1}^{M}\sum_{i=1}^{N}\bm{y}_{\tau_{i},j}-\bm{x}_{0}\right\|_{2}^{2}\nonumber\\
    &\leq \frac{2}{MN}\sum_{j=1}^{M}\sum_{i=1}^{N}\left\|\widehat{f}(\bm{y}_{0,j},\tau_{i})-\bm{y}_{\tau_{i},j}\right\|_{2}^{2}+2\left\|\frac{1}{MN}\sum_{j=1}^{M}\sum_{i=1}^{N}\bm{y}_{\tau_{i},j}-\bm{x}_{0}\right\|_{2}^{2}.
\end{align}
Let $B=\left(L\sqrt{B_{1}^{2}+1}+B_{2}\right)^{2}$.
By Lemma~\ref{lem:uniform boundedness} and Lemma~\ref{lem:uniform convergence for u-statistics}, the following inequality holds with probability at least $1-2\delta$
\begin{equation}
\label{eq:proof theorem 1 eq 6}
    \begin{split}
        &\frac{1}{MN}\sum_{j=1}^{M}\sum_{i=1}^{N}\left\|f^{*}(\bm{y}_{0,j},\tau_{i})-\bm{y}_{\tau_{i},j}\right\|_{2}^{2}\\
        &\leq \mathbb{E}[\|f^{*}(\bm{y}_{0},\tau)-\bm{y}_{\tau}\|_{2}^{2}]+\frac{2}{M}\sum_{j=1}^{M}\mathcal{R}_{j}(\mathcal{F})+\sqrt{\frac{B^{2}\log(M/\delta)}{2N}}+\sqrt{\frac{2B^{2}\log(1/\delta)}{2M}}.
    \end{split}
\end{equation}
Since $\widehat{f}$ is the empirical risk minimizer, \eqref{eq:proof theorem 1 eq 6} yields the following inequality that holds with probability at least $1-2\delta$:
\begin{equation}
\label{eq:proof theorem 1 eq 6.2}
    \begin{split}
        &\frac{1}{MN}\sum_{j=1}^{M}\sum_{i=1}^{N}\left\|\widehat{f}(\bm{y}_{0,j},\tau_{i})-\bm{y}_{\tau_{i},j}\right\|_{2}^{2}\\
        &\leq \mathbb{E}[\|f^{*}(\bm{y}_{0},\tau)-\bm{y}_{\tau}\|_{2}^{2}]+\frac{2}{M}\sum_{j=1}^{M}\mathcal{R}_{j}(\mathcal{F})+\sqrt{\frac{B^{2}\log(M/\delta)}{2N}}+\sqrt{\frac{B^{2}\log(1/\delta)}{2M}}.
    \end{split}
\end{equation}
Here we rearrange $\mathcal{R}(\mathcal{F})$ to obtain a bound,
\begin{equation}
\label{eq:proof theorem 1 eq 7}
    \begin{split}
        &\mathcal{R}_{j}(\mathcal{F})\\
        &=\mathbb{E}_{\tau,\sigma}\left[\sup_{f\in\mathcal{F}}\frac{1}{N}\sum_{i=1}^{N}\sigma_{i}\left(\|f(\bm{y}_{0,j},\tau_{i})\|_{2}^{2}-2\langle f(\bm{y}_{0,j},\tau_{i}),\bm{y}_{\tau_{i}}\rangle_{2}+\|\bm{y}_{\tau_{i}}\|_{2}^{2}\right)\right]\\
        &\leq\!
        \mathbb{E}_{\tau,\sigma}\!\left[\!\sup_{f\in\mathcal{F}}\frac{1}{N}\!\sum_{i=1}^{N}\!\sigma_{i}\!\left(\|f(\bm{y}_{0,j},\tau_{i})\|_{2}^{2}\!+\!\|\bm{y}_{\tau_{i}}\|_{2}^{2}\right)\!+\!2\sup_{f\in\mathcal{F}}\frac{1}{N}\!\sum_{i=1}^{N}\!-\sigma_{i}\langle f(\bm{y}_{0,j},\tau_{i}),\bm{y}_{\tau_{i}}\rangle_{2}\!\right]\\
        &=\!
        \mathbb{E}_{\tau,\sigma}\!\left[\!\sup_{f\in\mathcal{F}}\frac{1}{N}\sum_{i=1}^{N}\!\sigma_{i}\left(\|f(\bm{y}_{0,j},\tau_{i})\|_{2}^{2}\!+\!\|\bm{y}_{\tau_{i}}\|_{2}^{2}\right)\!+\!2\sup_{f\in\mathcal{F}}\frac{1}{N}\sum_{i=1}^{N}\sigma_{i}\langle f(\bm{y}_{0,j},\tau_{i}),\bm{y}_{\tau_{i}}\rangle_{2}\!\right]\!,
    \end{split}
\end{equation}
where we note that in \eqref{eq:proof theorem 1 eq 7} we utilize the obvious fact that $\sigma_{i}$ and $-\sigma_{i}$ follow the same distribution.
Since $\mathbb{E}_{\sigma_{i}}[\sigma_{i}]=0$ from the definition of $\sigma_{i}$, we have
\begin{equation}
\label{eq:proof theorem 1 eq 7.1}
    \begin{split}
        &\mathbb{E}_{\tau,\sigma}\left[\sup_{f\in\mathcal{F}}\frac{1}{N}\sum_{i=1}^{N}\sigma_{i}\left(\|f(\bm{y}_{0,j},\tau_{i})\|_{2}^{2}+\|\bm{y}_{\tau_{i}}\|_{2}^{2}\right)\right]\\
        &\leq \mathbb{E}_{\tau,\sigma}\left[\sup_{f\in\mathcal{F}}\frac{1}{N}\sum_{i=1}^{N}\sigma_{i}\|f(\bm{y}_{0,j},\tau_{i})\|_{2}^{2}\right].
    \end{split}
\end{equation}
Then, by Talagland's contraction inequality (see e.g., Lemma~5.7 in~\cite{mohri2018foundations}) to the first term in the right-hand-side of~\eqref{eq:proof theorem 1 eq 7.1}, it holds that
\begin{equation}
\label{eq:proof theorem 1 eq 7.2}
    \begin{split}
        &\mathbb{E}_{\tau,\sigma}\left[\sup_{f\in\mathcal{F}}\frac{1}{N}\sum_{i=1}^{N}\sigma_{i}\|f(\bm{y}_{0,j},\tau_{i})\|_{2}^{2}\right]\\
        &\leq 2L\sqrt{B_{1}^{2}+1}\cdot \mathbb{E}_{\tau,\sigma}\left[\sup_{f\in\mathcal{F}}\frac{1}{N}\sum_{i=1}^{N}\sigma_{i}\|f(\bm{y}_{0,j},\tau_{i})\|_{2}\right].
    \end{split}
\end{equation}

By Lemma~\ref{lem:rademacher complexity evaluation}, Lemma~\ref{lem:rademacher complexity evaluation version 2}, \eqref{eq:proof theorem 1 eq 4}, \eqref{eq:proof theorem 1 eq 6.2}, \eqref{eq:proof theorem 1 eq 7}, \eqref{eq:proof theorem 1 eq 7.1}, and \eqref{eq:proof theorem 1 eq 7.2}, with probability at least $1-2\delta$ we have
\begin{equation}
\label{eq:proof theorem 1 eq 8}
    \begin{split}
        &\frac{1}{MN}\sum_{j=1}^{M}\sum_{i=1}^{N}\left\|\widehat{f}(\bm{y}_{0,j},\tau_{i})-\bm{y}_{\tau_{i},j}\right\|_{2}^{2}\\
        &\leq \mathbb{E}[\|f^{*}(\bm{y}_{0},\tau)-\bm{y}_{\tau}\|_{2}^{2}]+4C(L\sqrt{B_{1}^{2}+1}+1)\frac{1}{\sqrt{N}}\\
        &\quad+3\sqrt{\frac{B^{2}\log(M/\delta)}{2N}}+\sqrt{\frac{B^{2}\log(1/\delta)}{2M}}.
    \end{split}
\end{equation}

Next, we upper bound the second term in~\eqref{eq:proof theorem 1 eq 4} as
\begin{align}
    &\left\|\frac{1}{MN}\sum_{j=1}^{M}\sum_{i=1}^{N}\bm{y}_{\tau_{i},j}-\bm{x}_{0}\right\|_{2}^{2}\nonumber\\
    &\leq 2\left\|\frac{1}{MN}\sum_{j=1}^{M}\sum_{i=1}^{N}\bm{y}_{\tau_{i},j}-\mathbb{E}[\bm{y}_{\tau}|\bm{x}_{0}]\right\|_{2}^{2}
    +\frac{2}{MN}\sum_{j=1}^{M}\sum_{i=1}^{N}\|\mathbb{E}[\bm{y}_{\tau_{i},j}|\bm{x}_{0}]-\bm{x}_{0}\|_{2}^{2}\nonumber\\
    \label{eq:proof theorem 1 eq 9}
    &=2\left\|\frac{1}{MN}\sum_{j=1}^{M}\sum_{i=1}^{N}\bm{y}_{\tau_{i},j}-\mathbb{E}[\bm{y}_{\tau}|\bm{x}_{0}]\right\|_{2}^{2}+\frac{2}{MN}\sum_{j=1}^{M}\sum_{i=1}^{N}\|\mathbb{E}_{\tau_{i}}[\phi_{\tau_{i}}(\bm{x}_{0})]-\bm{x}_{0}\|_{2}^{2}\\
    \label{eq:proof theorem 1 eq 10}
    &=2\left\|\frac{1}{MN}\sum_{j=1}^{M}\sum_{i=1}^{N}\bm{y}_{\tau_{i},j}-\mathbb{E}[\bm{y}_{\tau}|\bm{x}_{0}]\right\|_{2}^{2}+2\|\mathbb{E}_{\tau}[\phi_{\tau}(\bm{x}_{0})]-\bm{x}_{0}\|_{2}^{2}\\
    \label{eq:proof theorem 1 eq 10.5}
    &\leq 2\left\|\frac{1}{MN}\sum_{j=1}^{M}\sum_{i=1}^{N}\bm{y}_{\tau_{i},j}-\mathbb{E}[\bm{y}_{\tau}|\bm{x}_{0}]\right\|_{1}^{2}+2\|\mathbb{E}_{\tau}[\phi_{\tau}(\bm{x}_{0})]-\bm{x}_{0}\|_{2}^{2},
\end{align}
where in \eqref{eq:proof theorem 1 eq 9} we use the condition~\textbf{(A)} in Assumption~1.
Here we note the condition in the statement that $\|\bm{y}_{\tau_{i},j}\|_{2}\leq B_{2}$ holds almost surely.
Here note that by Jensen's inequality
\begin{align}
\label{eq:proof theorem 1 eq 10.525}
    \left|\frac{1}{MN}\sum_{j=1}^{M}\sum_{i=1}^{N}\bm{y}_{\tau_{i},j,k}-\mathbb{E}[\bm{y}_{\tau,k}|\bm{x}_{0}]\right|\leq 
    \frac{1}{M}\sum_{j=1}^{M}\left|\frac{1}{N}\sum_{i=1}^{N}\bm{y}_{\tau_{i},j,k}-\mathbb{E}[\bm{y}_{\tau,j,k}|\bm{x}_{0}]\right|.
\end{align}
By Hoeffding's inequality (see e.g., Theorem~D.2 in~\cite{mohri2018foundations}), with probability at least $1-2(Md)^{-1}\delta$ it holds that
\begin{align}
\label{eq:proof theorem 1 eq 10.55}
\left|\frac{1}{N}\sum_{i=1}^{N}\bm{y}_{\tau_{i},j,k}-\mathbb{E}[\bm{y}_{\tau,j,k}|\bm{x}_{0}]\right|\leq \sqrt{\frac{B_{2}^{2}\log(2Md/\delta)}{2N}}.
\end{align}
By~\eqref{eq:proof theorem 1 eq 10.525} and \eqref{eq:proof theorem 1 eq 10.55}, with probability at least $1-2\delta$ it holds that
\begin{align}
\label{eq:proof theorem 1 eq 10.555}
\left\|\frac{1}{MN}\sum_{j=1}^{M}\sum_{i=1}^{N}\bm{y}_{\tau_{i},j}-\mathbb{E}[\bm{y}_{\tau}|\bm{x}_{0}]\right\|_{1}^{2}\leq \frac{d^{2}B_{2}^{2}\log(2Md/\delta)}{2N}.
\end{align}

Finally, combining \eqref{eq:proof theorem 1 eq 1}, \eqref{eq:proof theorem 1 eq 3}, \eqref{eq:proof theorem 1 eq 4}, \eqref{eq:proof theorem 1 eq 8}, \eqref{eq:proof theorem 1 eq 10.5}, and \eqref{eq:proof theorem 1 eq 10.555}, with probability at least $1-4\delta$, we have
\begin{align*}
    \|\widehat{\bm{x}}_{0}-\bm{x}_{0}\|_{2}^{2}&\leq
    2L^{2}\mathbb{E}[\tau^{2}]+2\mathbb{E}[\|f^{*}(\bm{y}_{0},\tau)-\bm{y}_{\tau}\|_{2}^{2}]+2\|\mathbb{E}_{\tau}[\phi_{\tau}(\bm{x}_{0})]-\bm{x}_{0}\|_{2}^{2}\\
    &\quad
    +2L^{2}\sqrt{\frac{\log(1/\delta)}{2N}}+4C(L\sqrt{B_{1}^{2}+1}+1)\frac{1}{\sqrt{N}}\\
    &\quad\quad+6\sqrt{\frac{B^{2}\log(M/\delta)}{2N}}+2\sqrt{\frac{B^{2}\log(1/\delta)}{2M}}+\frac{d^{2}B_{2}^{2}\log(2Md/\delta)}{2N}
    ,
\end{align*}
which shows the claim, and the proof is completed.
\end{proof}

\subsection{Proof of Lemma~\ref{lem:uniform convergence for u-statistics}}
\label{appsubsec:proof of lemma uniform convergence for u-statistics}
\begin{proof}[Proof of Lemma~\ref{lem:uniform convergence for u-statistics}]    
We first focus on each of the components in the decomposition
\begin{align}
\label{eq:stat bound eq 1}
\frac{1}{MN}\sum_{j=1}^{M}\sum_{i=1}^{N}\|f(\bm{y}_{0,j},\tau_{i})-\bm{y}_{\tau_{i},j}\|_{2}^{2}
=
\frac{1}{M}\sum_{j=1}^{M}\left(\frac{1}{N}\sum_{i=1}^{N}\|f(\bm{y}_{0,j},\tau_{i})-\bm{y}_{\tau_{i},j}\|_{2}^{2}\right).
\end{align}
In the following, we treat $\bm{y}_{\tau_{i},j}$ as a functional of $\tau_{i}$, namely we fix the randomness of $\bm{y}_{t,j}$ by considering a realization of it.
Here to avoid any confusion, we write $\bm{y}_{j}(\tau_{i})$ instead to express that it is a functional.
Then, the sum $N^{-1}\sum_{i=1}^{N}\|f(\bm{y}_{0,j},\tau_{i})-\bm{y}_{\tau_{i},j}\|_{2}^{2}$ is a mean of the function $\|f(\bm{y}_{0},\tau)-\bm{y}_{\tau}\|_{2}^{2}$ with i.i.d. random variables $(\tau_{1},\cdots,\tau_{N})$.
Thus, by Theorem~3.3 in~\cite{mohri2018foundations}, with probability at least $1-\delta/M$ we have
\begin{equation}
\label{eq:stat bound eq 2}
    \begin{split}
        &\frac{1}{N}\sum_{i=1}^{N}\|f(\bm{y}_{0,j},\tau_{i})-\bm{y}_{j}(\tau_{i})\|_{2}^{2}\\
        &\leq \mathbb{E}_{\tau}\left[\|f(\bm{y}_{0,j},\tau_{i})-\bm{y}_{j}(\tau_{i})\|_{2}^{2}\right]+2\mathcal{R}_{j}(\mathcal{F})+\sqrt{\frac{B^{2}\log(M/\delta)}{2N}}.
    \end{split}
\end{equation}
Since \eqref{eq:stat bound eq 2} holds for every $j\in\{1,\cdots,M\}$, by taking the union bound, with probability at least $1-\delta$ we have
\begin{equation}
\label{eq:stat bound eq 3}
    \begin{split}
        &\frac{1}{MN}\sum_{j=1}^{M}\sum_{i=1}^{N}\|f(\bm{y}_{0,i},\tau_{i})-\bm{y}_{j}(\tau_{i})\|_{2}^{2}\\
        &\leq \frac{1}{M}\sum_{j=1}^{M}\mathbb{E}_{\tau}\left[\|f(\bm{y}_{0,j},\tau_{i})-\bm{y}_{j}(\tau_{i})\|_{2}^{2}\right]+\frac{2}{M}\sum_{j=1}^{M}\mathcal{R}_{j}(\mathcal{F})+\sqrt{\frac{B^{2}\log(M/\delta)}{2N}}.
    \end{split}
\end{equation}
The equation \eqref{eq:stat bound eq 3} holds for any possible functional $\bm{y}_{j}$ in this problem setting.
By Hoeffding's inequality (see e.g., Theorem~D.2 in~\cite{mohri2018foundations}), with probability at least $1-\delta$ it holds that
\begin{align}
\label{eq:stat bound eq 4}
&\frac{1}{M}\sum_{j=1}^{M}\mathbb{E}_{\tau}\left[\|f(\bm{y}_{0,j},\tau_{i})-\bm{y}_{j}(\tau_{i})\|_{2}^{2}\right]\nonumber\\
&\leq \mathbb{E}_{\bm{y},\tau}[\|f(\bm{y}_{0,j},\tau_{i})-\bm{y}_{j}(\tau_{i})\|_{2}^{2}]+\sqrt{\frac{B^{2}\log(1/\delta)}{2M}}.
\end{align}
Combining \eqref{eq:stat bound eq 3} and \eqref{eq:stat bound eq 4}, we obtain the claim.
\end{proof}

\subsection{Proofs of Lemma~\ref{lem:rademacher complexity evaluation} and Lemma~\ref{lem:rademacher complexity evaluation version 2}}
\label{appsubsec:proof of lemma rademacher complexity evaluation}

The proof of Lemma~\ref{lem:rademacher complexity evaluation} uses the following useful facts.
Let $N_{0}\in\mathbb{N}$ be an arbitrary integer, and let $\mathcal{S}$ be an arbitrary measurable space.
Let $r_{1},\cdots,r_{N_{0}}$ be i.i.d. $\mathcal{S}$-valued random variables drawn from some probability distribution, and let $\mathcal{G}$ be a set of real-valued functions on $\mathcal{S}$.
In \cite[Definition~3.1]{mohri2018foundations}, a version of \emph{empirical Rademacher complexity}~\cite{bartlett2002rademacher} is defined as
\begin{align}
\label{eq:empirical rademacher complexity}
\widehat{\mathcal{R}}(\mathcal{G})&:=\mathbb{E}_{\sigma}\left[\sup_{g\in\mathcal{G}}\frac{1}{N_{0}}\sum_{i=1}^{N_{0}}\sigma_{i}g(r_{i})\right].
\end{align}
Denote $\|g\|_{N_{0}}^{2}=\frac{1}{N_{0}}\sum_{i=1}^{N_{0}}g^{2}(r_{i})$.
Also, $\mathcal{N}(\mathcal{G},u,\|\cdot\|_{N_{0}})$ denotes the standard covering number of $\mathcal{G}$ with ball coverings of radius $u$ in the norm $\|\cdot\|_{N_{0}}$ (see e.g., Definition~4.2.2 in \cite{vershynin2018high} for covering numbers).

It is shown by Theorem~18 in~\cite{von2004distance} that the following fact holds:
\begin{lem}[Special case of Theorem~18 in~\cite{von2004distance}]
\label{citelem:von lemma}
Let $\mathcal{S}=[0,1]$, and let $\rho$ be the standard Euclidean distance in $\mathcal{S}$.
Let $\widetilde{\mathcal{G}}$ be a set of real-valued functions on $\mathcal{S}$ such that every element is $1$-Lipschitz with respect to the metric $\rho$.
Then, there exists a constant $C>0$ independent of $N_{0}$ such that it holds that
\begin{align*}
\widehat{\mathcal{R}}(\widetilde{\mathcal{G}})\leq C\inf_{c>0}\left\{c+\frac{1}{\sqrt{N_{0}}}\int_{c/4}^{2\kappa(\mathcal{S})}\sqrt{\mathcal{N}(\mathcal{S},u/2,\rho)+\log \left(\frac{2\kappa(\mathcal{S})}{u}+1\right)}du\right\},
\end{align*}
where $\kappa(\mathcal{S})$ denotes the diameter of $\mathcal{S}$.
\end{lem}
Note that \cite{von2004distance} show a more general case where $\mathcal{S}$ is a totally bounded metric space in the theorem.

Using the above fact, we now prove Lemma~\ref{lem:rademacher complexity evaluation} and Lemma~\ref{lem:rademacher complexity evaluation version 2}.
\begin{proof}[Proof of Lemma~\ref{lem:rademacher complexity evaluation}]
    Define the function class $\mathcal{G}=\{g(t)=\|f(\bm{y}_{0,j},t)\|_{2}|f\in\mathcal{F}\}$.
    We rescale every function $g$ in $\mathcal{G}$ to define
    \begin{align*}
        \widetilde{\mathcal{G}}=\left\{\left.\widetilde{g}(t)=L^{-1}d^{-1}g(t)\right|g\in\mathcal{G}\right\}.
    \end{align*}
    where $\bm{u}\in U$, and $t\in[0,1]$.
    Denote $f=(f_{1},\cdots,f_{d})$ for convenience. Here, we notice that for any $\widetilde{g}\in\widetilde{\mathcal{G}}$,
    \begin{align}
        |\widetilde{g}(t)-\widetilde{g}(t')|&\leq L^{-1}d^{-1}\|f(\bm{y}_{0,j},t)-f(\bm{y}_{0,j},t')\|_{2} \nonumber\\
        \label{eq:norm bound}
        &\leq |t-t'|,
    \end{align}
    where in~\eqref{eq:norm bound} we use the $L$-Lipschitzness of $f$.
    Therefore, by Lemma~\ref{citelem:von lemma} due to~\cite{von2004distance}, there exists some constant $C>0$ independent of $N$ such that
    \begin{align}
        \label{eq:von luxburg chaining bound}
            \widehat{\mathcal{R}}(\mathcal{G})
            \leq LC\inf_{c>0}\left\{c+\frac{1}{\sqrt{N}}\int_{c/4}^{2}\sqrt{\mathcal{N}([0,1],u/2,\|\cdot\|_{2})+\log 2u^{-1}}du\right\}.
    \end{align}

    In the remaining part, we follow a similar way to Example~4 in~\cite{von2004distance}.
    By Corollary 4.2.13 in~\cite{vershynin2018high}, for the quantity $\mathcal{N}([0,1],u/2,\|\cdot\|_{2})$ there exists a constant $C>0$ independent of $u$ such that
    \begin{align}
        \label{eq:euclidean ball bound}
        \mathcal{N}([0,1],u/2,\|\cdot\|_{2})\leq \frac{C}{u}.
    \end{align}
    We can bound the following integral:
    \begin{align}
    \label{eq:integral bound}
            \int_{c/4}^{2}\sqrt{\frac{1}{u}+\log\left(\frac{2}{u}\right)}du
            &\leq \int_{c/4}^{2}\sqrt{\frac{1}{u}}du+\int_{c/4}^{2}\sqrt{\log\left(\frac{2}{u}\right)}du\nonumber\\
            &\lesssim 1,
    \end{align}
    where the notation $a_{1}\lesssim a_{2}$ ($a_{1},a_{2}\in\mathbb{R}$) means that $a_{1}\leq C a_{2}$ for some constant $C>0$.
    Combining \eqref{eq:von luxburg chaining bound}, \eqref{eq:euclidean ball bound}, \eqref{eq:integral bound}, we have
    \begin{align*}
        \widehat{\mathcal{R}}(\mathcal{G})\lesssim \inf_{c>0}\left\{c+\frac{1}{\sqrt{N}}\right\}.
    \end{align*}
    Choosing as $c=1/\sqrt{N}$, we obtain $\widehat{\mathcal{R}}(\mathcal{G})\lesssim 1/\sqrt{N}$.
    This bound is the displayed inequality of this lemma.
\end{proof}

\begin{proof}[Proof of Lemma~\ref{lem:rademacher complexity evaluation version 2}]
The proof of Lemma~\ref{lem:rademacher complexity evaluation version 2} is almost the same as that of Lemma~\ref{lem:rademacher complexity evaluation}.
In fact, since $\|\bm{y}_{\tau_{i},j}\|_{2}$ is bounded from the assumption, we have
\begin{align*}
    &\mathbb{E}_{\tau,\sigma}\left[\sup_{f\in\mathcal{F}}\frac{1}{N}\sum_{i=1}^{N}\sigma_{i}\langle f(\bm{y}_{0,j},\tau_{i}),\bm{y}_{\tau_{i},j}\rangle_{2}\right]\\
    &\leq B_{2}\sum_{k=1}^{d}\mathbb{E}_{\tau,\sigma}\left[\sup_{f_{k}\in\mathcal{F}_{k}}\frac{1}{N}\sum_{i=1}^{N}\sigma_{i}(B_{2}^{-1}\bm{y}_{\tau_{i},j,k})f_{k}(\bm{y}_{0,j},\tau_{i})\right],
\end{align*}
where $\mathcal{\mathcal{F}}_{k}=\{f_{k}|f=(f_{1},\cdots,f_{k},\cdots,f_{d})\in\mathcal{F}\}$, and $\bm{y}_{\tau_{i},j}=(\bm{y}_{\tau_{i},j,1},\cdots,\bm{y}_{\tau_{i},j,d})$.
Since $B_{2}^{-1}\bm{y}_{\tau_{i},j,k}\leq 1$ from the assumption that $\|\bm{y}_{\tau_{i},j}\|_{2}\leq B_{2}$ holds almost surely, the rescaled function $(B_{2}^{-1}\bm{y}_{\tau_{i},j,k})f_{k}(\cdot,\cdot)$ is also $L$-Lipschitz continuous.
By Theorem~12 in~\cite{bartlett2002rademacher} which shows the monotonicity of $\widehat{\mathcal{R}}$, we have
\begin{align*}
    \mathbb{E}_{\sigma}\left[\sup_{f_{k}\in\mathcal{F}_{k}}\frac{1}{N}\sum_{i=1}^{N}\sigma_{i}(B_{2}^{-1}\bm{y}_{\tau_{i},j,k})f_{k}(\bm{y}_{0,j},\tau_{i})\right]\leq \widehat{\mathcal{R}}(\mathcal{F}_{k}).
\end{align*}
Therefore, by Lemma~\ref{citelem:von lemma} due to~\cite{von2004distance}, Corollary~4.2.13 in~\cite{vershynin2018high}, and some calculation based on Example~4 in~\cite{von2004distance}, as in the proof of Lemma~\ref{lem:rademacher complexity evaluation}, we obtain the claim.
\end{proof}

\section{Details of Numerical Experiments by N2N and TD}
\label{append: further expt} 
For implementation of Noise2Noise (N2N)~\cite{lehtinen2018noise}, we need to define a noisy input image $\bm{y}$ and the noisy target $\bm{y}^\prime$. In our preliminary experiments, we defined the pair $(\bm{y}, \bm{y}^\prime)$ based on the GitHub code~\cite{lehtinen2018n2nGit}. However, the prediction results were not satisfactory; see Table~\ref{tab:toydata results of naive n2n}, Table~\ref{tab:acdc results of navie n2n}, and Figure~\ref{fig: mfm and cryoem before after by naive n2n}. The reason is likely the limited number of training images. Therefore, inspired by Topaz-Denoise (TD)~\cite{bepler2020topaz}, which empirically performs well for limited amount of training images, we define the pair in Section~\ref{subsubsec: expt1 setting} to satisfy the following condition: $|t - t^\prime| = 1$, where $t$ and $t'$ denote the time indexes of $\bm{y}$ and $\bm{y}^\prime$, respectively. 

\begin{table*}[h]
\caption{
Predicting performance by N2N on toy image datasets used in Expt.1.
}
\label{tab:toydata results of naive n2n}
    \centering
    \scalebox{0.8}{
        \begin{tabular}{ccccccccc}
        \toprule
             \multirow{2}{*}{$(\lambda,\sigma)$} & \multicolumn{2}{c}{$(25,10)$}  & \multicolumn{2}{c}{$(10,10)$}  & \multicolumn{2}{c}{$(25,25)$}  & \multicolumn{2}{c}{$(10,25)$}   \\ \cline{2-9}
             & PSNR$\uparrow$ & SSIM$\uparrow$ & PSNR & SSIM & PSNR & SSIM & PSNR & SSIM \\ \hline
            Slow & 18.82 & 0.841 & 18.25 & 0.775 & 18.51 & 0.856 & 19.10 & 0.879 \\
            Fast & 22.69 & 0.905 & 18.04 & 0.849 & 20.29 & 0.854 & 21.72 & 0.843 \\
            \bottomrule
        \end{tabular}
    }
\end{table*}
\begin{table*}[h]
\caption{
    Predicting performance by N2N on ACDC image datasets used in Expt.2. 
}
\label{tab:acdc results of navie n2n}
    \centering
    \scalebox{0.8}{
        \begin{tabular}{cccccccccc}
        \toprule
             \multicolumn{2}{c}{ID 1}  & \multicolumn{2}{c}{ID 2}  & \multicolumn{2}{c}{ID 3}  & \multicolumn{2}{c}{ID 4}  & \multicolumn{2}{c}{ID 5} \\ \cline{1-10}
             PSNR$\uparrow$ & SSIM$\uparrow$ & PSNR & SSIM & PSNR & SSIM & PSNR & SSIM & PSNR & SSIM \\ \hline
             19.40 & 0.641 & 20.89 & 0.633 & 20.65 & 0.650 & 16.12 & 0.575 & 19.32 & 0.599 \\
            \bottomrule
        \end{tabular}
    }
\end{table*}
\begin{figure*}[h]
    \begin{tabular}{ccc}
            \begin{subfigure}[b]{.3\linewidth}
            \centering
                \includegraphics[width=0.6\hsize]{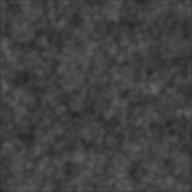}
                \caption{}
                \label{subfig: cryoem-top-left-naive-n2n-pred}
            \end{subfigure} &
            \begin{subfigure}[b]{.3\linewidth}
            \centering
                \includegraphics[width=0.6\hsize]{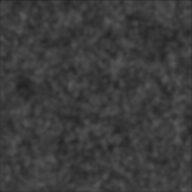}
                \caption{}
                \label{subfig: cryoem-bottom-right-naive-n2n-pred}
            \end{subfigure} &
            \begin{subfigure}[b]{.3\linewidth}
            \centering
                \includegraphics[width=0.6\hsize]{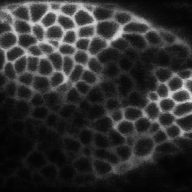}
                \caption{}
                \label{subfig: mfm-naive-n2n-pred}
            \end{subfigure} 
    \end{tabular}
    \caption{
    Visualization of predicted Cryo-EM and FM images by N2N.
    \eqref{subfig: cryoem-top-left-naive-n2n-pred} and~\eqref{subfig: cryoem-bottom-right-naive-n2n-pred}: The predicted Cryo-EM images corresponding to \eqref{subfig: original-top-left} and~\eqref{subfig: original-bottom-right}, respectively,
    \eqref{subfig: mfm-naive-n2n-pred}: The predicted FM image corresponding to~\eqref{subfig: mfm-clipped-noisy-t0}.
    }
    \label{fig: mfm and cryoem before after by naive n2n} 
\end{figure*}

\bibliographystyle{unsrtnat}
\bibliography{allref_arxiv}

\end{document}